\documentclass[numbers]{article}

\usepackage[main, final]{neurips_2026}

\usepackage[utf8]{inputenc} 
\usepackage[T1]{fontenc}    
\usepackage{hyperref}       
\usepackage{url}            
\usepackage{booktabs}       
\usepackage{amsfonts}       
\usepackage{nicefrac}       
\usepackage{microtype}      
\usepackage{xcolor}         
\usepackage{amsmath}
\usepackage{amsthm}
\usepackage{amssymb}
\usepackage{graphicx}
\graphicspath{{../}{../images/}{./}{images/}}
\usepackage{subcaption}
\usepackage{booktabs}
\usepackage{algorithm}
\usepackage{algpseudocode}
\usepackage{float}
\usepackage{enumitem}
\usepackage{wrapfig}
\usepackage{relsize}
\usepackage{bbm}
\usepackage{tikz}
\usepackage{multirow}
\usepackage{array}
\usetikzlibrary{
  arrows.meta,
  positioning,
  calc,
  decorations.pathreplacing,
  backgrounds,
  fit
}
\newtheorem{definition}{Definition}

\newtheorem{theorem}{Theorem}[section]

\newtheorem{lemma}[theorem]{Lemma}

\def\vx{\mathbf{x}}
\def\vX{\mathbf{X}}

\def\vE{\mathbf{E}}
\def\vb{\mathbf{b}}
\def\vp{\mathbf{p}}
\def\ER{{Erd\H{o}s--R\'enyi}}
\def\std{\textup{\texttt{std}}}
\newcommand{\norm}[1]{\left\lVert #1 \right\rVert}
\newcommand{\normFn}[1]{\left\lVert #1 \right\rVert_{F,n}}
\title{When Edge Independence Fails: Joint Graph Diffusion with Latent Sociability Priors}

\author{%
  Adarsh Jamadandi$^{1}$, Nicolas Keriven$^{1}$, Aline Roumy$^{2}$ \\
     CNRS, IRISA, Rennes, France$^{1}$, INRIA, Rennes, France$^{2}$\\
\texttt{adarsh.jamadandi@irisa.fr}, \texttt{nicolas.keriven@cnrs.fr}, \texttt{aline.roumy@inria.fr}
}

\begin{document}

\maketitle

\begin{abstract}
Discrete graph diffusion models corrupt graph structure through independent edge updates with a Markovian noise model. Combined with the exchangeability requirement on the output distribution, the terminal prior is necessarily an Erd\H{o}s--R\'enyi (E-R) random graph. This has a fundamental consequence for unconditional graph generation without node features: in the large graph limit, any GNN denoiser applied to an E-R graph approximately outputs an E-R graph, making effective denoising difficult. In practice, generated graphs fail to recover meaningful node heterogeneity and higher-order structure. Previous models introduced specialized node features to compensate, framing this as an expressivity problem, while the role of the prior remained unexamined. We instead address the root cause by introducing the \emph{Latent Sociability Prior} (LSP) and a \emph{joint} diffusion process that co-evolves latent node sociability variables and graph structure, removing edge independence while preserving exchangeability. Unlike existing latent graph diffusion methods, the sociability mechanism requires no complex encoder: its sole role is to guide edge updates and preserve node heterogeneity throughout the trajectory. We validate our approach on multiple graph generation datasets and show that the proposed framework more accurately captures graph distributions than baseline methods without node features.
\end{abstract}

\section{Introduction}
Diffusion models \citep{pmlr-v37-sohl-dickstein15, song2020generativemodelingestimatinggradients,
ddpm, pmlr-v139-nichol21a} have recently emerged as a powerful class of generative models, with
major improvements in image \citep{zhang2023surveydiffusionbasedimage} and video generation
\citep{melnik2024videodiffusionmodelssurvey}. Motivated by this success, diffusion-based methods
have been extended to graph generation. The two most popular approaches are: embedding nodes into a continuous latent space
and applying Gaussian diffusion
\citep{pmlr-v108-niu20a, pmlr-v162-jo22a, evdaimon2024neuralgraphgeneratorfeatureconditioned},
or operating directly in the graph space through a discrete corruption process on edges
\citep{d3pm}. The latter approach has been shown to be more appropriate for generating realistic graph
structures \citep{haefeli2023diffusionmodelsgraphsbenefit,
vignac2023digressdiscretedenoisingdiffusion}, leading to a growing body of work on both
graph diffusion and discrete flow matching
\citep{gat2024discrete, qinmadeira2024defog, roos2026categoricalflowmaps}.

Despite this recent progress, discrete diffusion models also have inherent problems. For
instance, scaling to large graphs is a bottleneck as the computational complexity grows
quadratically with increasing graph size. Recent methods
\citep{chen2023efficientdegreeguidedgraphgeneration, li2024graphmaker, qin2025sparsediff} have
made notable progress on addressing these practical limitations. In this work, however, we
identify a more fundamental limitation at the core of existing graph diffusion models. We find
that, in unconditional graph generation without node features, the combination of edge
independence in the forward process and the requirement that the output graph distribution remains exchangeable
tightly constrains the terminal distribution, leaving Erd\H{o}s--R\'enyi (E-R) as the
\emph{only} admissible prior. We show theoretically and empirically that such graph diffusion
models exhibit a \emph{fixed-point} pathology in the large graph limit: for most architectures
including attention-based GNNs, when initialized from an E-R prior, \emph{any} denoiser
approximately outputs constant edge predictions (Theorem~\ref{thm:concentration_ER}) which is
also an E-R random graph. Consequently, the generated graphs tend to look homogeneous and lack
higher-order structure.

In practice, this failure is commonly attributed to insufficient model expressivity, with
auxiliary node features introduced as a remedy even in the unconditional setting \citep{vignac2023digressdiscretedenoisingdiffusion,chen2023efficientdegreeguidedgraphgeneration,qin2025sparsediff}, leaving the
role of the E-R prior unexamined. While such features may improve performance, they can be
expensive to compute and obscure the real source of failure. Our claim is therefore not simply
that existing models are insufficiently expressive, but that the standard edge-independent
discrete diffusion pipeline is fundamentally mismatched with unconditional graph generation.

As a resolution, we propose a \emph{joint diffusion} framework that co-evolves continuous
latent variables together with the discrete graph structure, combining the strengths of latent
diffusion models and graph-space discrete diffusion. Our framework is built around a new prior
inspired by inhomogeneous random graph models
\citep{Norros_Reittu_2006, Bollob_s_2007, Hofstad_2016, caron-fox}, termed the \emph{Latent
Sociability Prior} (LSP). Each node is endowed with a positive sociability weight, and edge
probabilities are conditioned on these weights, inducing node-level heterogeneity. We jointly
diffuse these variables and the graph edges during both noising and denoising, allowing them to
guide edge updates throughout the diffusion trajectory. As a consequence, the terminal prior is
no longer edge-independent and therefore does not induce the fixed-point behavior of standard
discrete denoisers. Crucially, the sociability mechanism is much simpler than existing latent
graph diffusion approaches, as its role is not to
generate graphs alone but to guide the companion discrete diffusion process. In particular, it requires no pre-trained encoder, and the clean latent variables
admit an approximate closed-form expression from the training data. 

Our contributions can be summarized as follows:
\begin{enumerate}
    \item We identify a fundamental limitation in existing discrete graph diffusion pipelines.
    Without node features, an edge-independent terminal prior together with independent edge
    updates leads to a fixed-point pathology: any denoiser applied to an E-R random graph
    approximately yields an E-R random graph, making it difficult to escape this regime and preventing recovery of meaningful node
    heterogeneity and higher-order structure.

    \item We propose a joint diffusion framework for graph generation that co-evolves continuous
    latent variables together with discrete graph structure, allowing latent node heterogeneity
    to directly guide edge updates throughout both noising and denoising.

    \item We instantiate this framework through the \emph{Latent Sociability Prior} (LSP), a
    novel exchangeable latent-variable prior inspired by inhomogeneous random graph models.
    The LSP admits a closed-form approximation for the clean latent variables, bypassing the
    need for an encoder, while inducing beneficial edge dependencies and preserving node-level
    heterogeneity.

    \item We validate our approach on multiple graph generation benchmarks and show that, even
    without node features, the proposed framework more accurately captures the target graph
    distribution compared to existing baselines.
\end{enumerate}

\section{Related work}
\paragraph{Latent graph diffusion.}
Extending the original diffusion formulation \citep{pmlr-v37-sohl-dickstein15,
song2020generativemodelingestimatinggradients, ddpm, pmlr-v139-nichol21a} for graph generation
involves embedding the graphs into a low-dimensional continuous space using variational
autoencoders (VAEs) and corrupting the graph structure using Gaussian noise in this latent space
\citep{pmlr-v108-niu20a, pmlr-v162-jo22a, evdaimon2024neuralgraphgeneratorfeatureconditioned}.
Although this approach is appealing since it allows for a computationally friendly diffusion
process, it also presents unique challenges. For instance, VAEs are notorious for training
instabilities \citep{unstablevaes} and small reconstruction errors in the encoding-decoding
process can affect the structural validity of generated graphs. Further, generating large graphs is challenging since the
dimensionality of the latent space itself becomes a bottleneck, limiting the structural detail that can be faithfully recovered.

\paragraph{Discrete diffusion for graphs.}
Another popular approach for graph diffusion is to adopt the discrete diffusion paradigm
\citep{d3pm} where the corruption process is carried out directly in the space of graphs through
independent edge additions and deletions. This framework has been shown to be more appropriate
for generating realistic graph structures, resulting in large improvements in both diffusion
models \citep{haefeli2023diffusionmodelsgraphsbenefit,
vignac2023digressdiscretedenoisingdiffusion, chen2023efficientdegreeguidedgraphgeneration,
li2024graphmaker, osman2025lgdclatentgraphdiffusion, ninniri2025graphdiffusioninsertdelete,
karami2025higenhierarchicalgraphgenerative, siraudin2026principled} and discrete flow matching
\citep{gat2024discrete, jiang2026bureswasserstein, qinmadeira2024defog,
roos2026categoricalflowmaps}. However, most of these advancements address practical issues such
as scalability, since the computational burden grows quadratically with increasing graph sizes.
In this work, we identify a more fundamental limitation specific to discrete graph diffusion
models: the edge-independent forward process forces the terminal prior to be Erd\H{o}s--R\'enyi,
which leads to a fixed-point pathology that constrains generation quality independently of model
capacity.

\paragraph{Influence of prior distributions.}
The role of the prior in graph generation has received some attention.
\citep{powerindependent, edgedependency} study the limitations of edge-independent generative
models and show they cannot simultaneously produce diverse graphs and capture higher-order
statistics such as triangle and cycle counts a theoretical finding that directly motivates our
work. In the continuous flow matching setting, \citep{chen2026priorinformedflowmatchinggraph}
use embedding-based source distributions as informed starting points, and
\citep{wijesinghe2026flowetteflowmatchinggraphette} propose motif-based priors via Gromov--Wasserstein-coupled flow matching, improving structural fidelity at the cost of expensive optimal transport and manual motif specification per graph family. These methods operate outside the discrete diffusion setting and do not address the fixed-point pathology we identify. Our approach requires no optimal transport and no domain-specific prior engineering: node heterogeneity emerges directly from the latent sociability variables by construction.

\section{Background on discrete denoising diffusion framework} \label{ddframework}
We recall the discrete denoising diffusion framework for graph generation following \citep{d3pm,vignac2023digressdiscretedenoisingdiffusion}. Let
$\mathcal{D}_{\mathrm{data}}=\{G_1,\dots,G_M\}$ be a dataset of graphs drawn from an unknown distribution $p_{\mathrm{data}}$. For a graph $G=(X,E)$ with node features $X$ and edges $E$, the goal is to learn a distribution $p_\theta(G)$ from which new graphs can be sampled. Node features and edges are represented via one-hot encodings, giving tensors $\mathbf{X} \in \mathbb{R}^{n \times d}$ and $\mathbf{E} \in \mathbb{R}^{n \times n \times b}$, where $d$ is the number of node feature types and $b$ the number of edge categories. Here we simply consider $b=2$ edge categories, with $e_{ij} \in \mathcal{B} = \{0,1\}$ denoting the scalar edge state (absence or presence) between nodes $i$ and $j$, and $\mathbf{e}_{ij} \in \mathbb{R}^b$ its one-hot encoding, organized into the edge tensor $\mathbf{E} \in \mathbb{R}^{n \times n \times b}$. 

Since we focus on unconditional graph generation \emph{without node features} we set $\mathbf{X}=\mathbf{1}$ throughout, unless otherwise written. The forward process defines a sequence of increasingly noisy graphs $G^0, G^1, \dots, G^T$,
where $G^0$ is a clean training graph and $G^T$ is drawn from the terminal prior $\pi_T$. We
denote by $e_{ij}^0 \in \mathcal{B}$ the clean edge state between nodes $i$ and $j$ in $G^0$,
and by $\mathbf{e}_{ij}^0 \in \mathbb{R}^b$ its one-hot encoding. The probability of a noisy
sequence $(G^1, \dots, G^T)$ given $G^0$ factorizes as
\begin{equation}
q(G^{1:T}\mid G^0)=\prod_{t=1}^T q(G^t\mid G^{t-1})
\label{eq:graph_markov_chain}
\end{equation}
The transition matrix $\mathbf{Q}_E^t\in\mathbb{R}^{b\times b}$ governs one-step edge corruption,
with entries $[\mathbf{Q}_E^t]_{rs}=q(e^t=s\mid e^{t-1}=r)$, and
$\bar{\mathbf{Q}}_E^t=\mathbf{Q}_E^1\cdots\mathbf{Q}_E^t$ is the cumulative transition matrix.
Since transitions act independently on each edge, the per-edge forward marginal is the $1\times b$
probability vector $q(e_{ij}^t \mid e_{ij}^0) = \mathbf{e}_{ij}^0\,\bar{\mathbf{Q}}_E^t$. The
joint forward marginal over all edges is
\begin{equation}
q(\mathbf{E}^t \mid \mathbf{E}^0) = \prod_{1 \leq i,j \leq n} q(e_{ij}^t \mid e_{ij}^0)
\label{eq:digress_forward_compact}
\end{equation}
For undirected graphs, noise is applied only to the upper-triangular part of $\mathbf{E}$ and
then symmetrized.
A denoising network $\phi_\theta$ takes a noisy graph $G^t=(\mathbf{1},\mathbf{E}^t)$ and
timestep $t$ as input, and outputs for each node pair $(i,j)$ a probability vector
$\hat{\mathbf{p}}_{ij}^{E,t}\in\mathbb{R}^b$ over edge types. It is optimized via cross-entropy
against the clean edge states:
\begin{equation}
\mathcal{L}_{\mathrm{CE}}(\hat{\mathbf{p}}^{E,t},G^0)
=
\sum_{1\le i,j\le n}\mathrm{CE}\!\left(e^0_{ij},\hat{\mathbf{p}}_{ij}^{E,t}\right)
\label{eq:digress_ce_edges}
\end{equation}
New graphs are sampled by running the reverse process from $G^T\sim\pi_T$. Since node features
remain fixed, the reverse distribution factorizes over edges:
\begin{equation}
p_\theta(\mathbf{E}^{t-1}\mid G^t)=\prod_{1\le i,j\le n} p_\theta(e_{ij}^{t-1}\mid G^t)
\label{eq:reverse_factorization}
\end{equation}
Each factor is computed by marginalizing over the predicted clean edge distribution:
\begin{equation}
p_\theta(e_{ij}^{t-1}\mid G^t)
=
\sum_{r\in\mathcal{B}}
q\!\left(e_{ij}^{t-1}\mid e_{ij}^t,e^0_{ij}=r\right)\hat{\mathbf{p}}_{ij}^{E,t}(r)
\label{eq:reverse_marginal}
\end{equation}
where $q(e_{ij}^{t-1}\mid e_{ij}^t, e_{ij}=r)$ is the forward posterior, available in closed
form in \citep{vignac2023digressdiscretedenoisingdiffusion} and $\hat{\mathbf{p}}_{ij}^{E,t}(r)$ denotes the network's estimate of the probability that the clean edge takes value $r$ given the noisy graph $G^t$. As we will see later, by Theorem~\ref{thm:concentration_ER},
when $G^t$ is an E-R random graph, $\hat{\mathbf{p}}_{ij}^{E,t}$ is approximately constant across all
pairs $(i,j)$, so $p_\theta(e_{ij}^{t-1}\mid G^t)$ is the same for all pairs and
$\mathbf{E}^{t-1}$ again follows an E-R distribution.

Transition matrices are chosen as in \citep{vignac2023digressdiscretedenoisingdiffusion} to
preserve the empirical marginal distribution of edge types $\mathbf{m}_E\in\mathbb{R}^b$
throughout diffusion through 
$\mathbf{Q}_E^t =
\alpha_t I + \beta_t \mathbf{1}_b \mathbf{m}_E^\top$ 
where, $\beta_t = 1-\alpha_t$.
With probability $\alpha_t$ the edge state is preserved, and with probability $\beta_t$ it is
resampled from $\mathbf{m}_E$. 

As $t\to T$, every edge is driven independently towards
$\mathbf{m}_E$, so the terminal prior $\pi_T := q(G^T)$ is a product of independent marginals,
an Erd\H{o}s--R\'enyi random graph where every edge is drawn independently with probability
equal to the empirical edge density of the training set.

\paragraph{Exchangeability.} An important feature of any graph generation method is that it
needs to be \emph{exchangeable}: all permutations of the generated graphs need to be
equiprobable. In other words, if we denote by $\sigma \cdot G$ the relabelling of a graph $G$
by a permutation $\sigma$, our entire generation process needs to satisfy
\begin{equation*}
    p_\theta(\sigma \cdot G^0) = p_\theta(G^0)
\end{equation*}
for all $\sigma, G^0$. Fortunately, if a) the prior distribution $\pi_T = q(G^T)$ is
exchangeable (which is naturally the case of E-R random graphs), and b) the denoiser
$\phi_\theta$ is \emph{permutation equivariant}\footnote{that is,
$\phi_\theta(\sigma \cdot G) = \sigma \cdot \phi_\theta(G)$}, then it is easy to see that the
final distribution is exchangeable: we have
$p_\theta(\sigma \cdot G^{t-1}\mid \sigma \cdot G^t) = p_\theta(G^{t-1}\mid G^t)$ by
equivariance, and $q(\sigma \cdot G^T) = q(G^T)$ by exchangeability of the prior. 

\section{Issue with \ER{} prior}
In this section, we show that the E-R prior introduces a fundamental limitation in the generation process: in the large-graph limit, E-R random graphs are approximate \emph{fixed points} for \emph{any} denoiser. That is, as the graph size grows, denoising an E-R graph approximately yields another E-R graph, making it increasingly difficult to escape this regime. We begin by asking whether a different prior could have been chosen within the discrete diffusion framework.

\paragraph{E-R prior is unavoidable.}
As we have seen in the previous section, the prior distribution $\pi_T$ needs to be exchangeable to guarantee exchangeability of the whole generation process. Moreover, the discrete diffusion process is based on successive \emph{independent} edge perturbations, such that the terminal distribution has fully independent edges. We then have the following result, shown in Appendix \ref{app:prooflemER} for completeness.
\begin{lemma}\label{lem:ER} \citep{ALDOUS1981581,hoover1979relations}
    Any \textbf{exchangeable} random graph model with \textbf{independent edges} is an E-R model.
\end{lemma}
In other words, as long as discrete diffusion operates with \emph{independent edges}, the terminal distribution $\pi_T$ is \emph{necessarily} an E-R random graph.

\paragraph{E-R is a fixed point for denoisers.}
Consider a simple denoiser that performs link prediction by first computing node representations with a GNN $\hat\vX = \phi_\theta(\vX, \vE)$ (recall that we generally consider $\vX=1$ in this paper, but the result below is more generally valid for node features $\vX$), then compute edge probabilities using these representations with a similarity function $\varphi$:
$$
\hat{p}_{ij} = \varphi(\vx_i^{(K)}, \vx_j^{(K)}) \in [0,1]
$$
This method is akin to a Variational Graph Auto-Encoder (VGAE) \cite{kingma2022autoencodingvariationalbayes,vgae}, but when the GNN is a GAT \citep{Velickovic:2018we}, this is close to the graph transformer \citep{dwivedi2021generalizationtransformernetworksgraphs} used in DiGress \cite{vignac2023digressdiscretedenoisingdiffusion}, removing some normalization layers and edge modulation for simplicity
\footnote{Technically, in vanilla graph transformers the edge probabilities would also be modulated by the input graph $\hat{p}_{ij} = \varphi(\vx_i^{(K)}, \vx_j^{(K)}, e_{ij})$, in which case our result become: $\hat{\vp}^E$ is approximately an \emph{affine function} of $\vE$ instead of being approximately constant. That is: the output graph is a probabilistic ``sum'' of the input E-R graph and another E-R graph. We choose our version for simplicity.}.
Denote by
\begin{equation}
    \mathsmaller{\std(\hat p) = \sqrt{\tfrac{1}{n^2}\sum_{ij}\left|\hat{p}_{ij} - \tfrac{1}{n^2}\sum_{ij} \hat{p}_{ij}\right|^2},\quad \std(\vX) = \sqrt{\tfrac{1}{n}\sum_i \norm{\vX_i - \tfrac{1}{n}\sum_i \vX_i}^2}}
\end{equation}
the standard deviation of the edge probabilities and input node features, respectively, such that E-R random graphs have exactly $\std(\hat p)=0$. We have the following result.
\begin{theorem}\label{thm:concentration_ER}
    Consider an input graph $\vE$ of size $n$, drawn from an \ER{} model with parameter $p$. Then, for the GNNs outlined below, there is $C$ such that: for all $\delta, \epsilon>0$, if $n$ is large enough, then with probability at least $1-\delta$,
    $$
        \std(\hat p) \lesssim C\std(\vX) + \epsilon
    $$
    where the multiplicative constant only depends on the parameters of the GNN. We have:
    \begin{itemize}
        \item if the GNN is a GAT \citep{Velickovic:2018we} with $K$ layers, then $C = (1/p)^{K}$, and $n$ is such that
        \begin{equation}
            p n \gtrsim \log(n/\delta)
        \end{equation}
        In particular, one can choose $\epsilon=0$ in this case.
        \item if the GNN is a GCN \citep{Kipf:2017tc}, then $C = O(1)$, and $n$ is such that
        \begin{equation}
            p n \gtrsim \epsilon^{-2}\log(n/\delta)
        \end{equation}
        \item if the GNN is GIN \citep{xu2018how}, for all $r>0$ there exists $C_r$ such that the result is valid with $C = O(1)$, and $n$ such that
        \begin{equation}
            p n \gtrsim \epsilon^{-2}\log(n/\delta) \quad \text{and}\quad n\geq C_r\delta^{-r}
        \end{equation}
    \end{itemize}
\end{theorem}
The details for this theorem and precise description of the considered GNNs are given in Appendix \ref{thmproofs}. In short, we thus have the following: when the node features have low standard deviation, the denoiser outputs approximately constant edge probabilities. Constant features are an extreme version of that, since $\std(\mathbf{1})=0$. In other words, the class of E-R random graphs represents an approximate \emph{fixed-point} for denoisers, and
as a result node heterogeneity remains low during the whole diffusion process (Figure~\ref{fig:propexpts}), as all graphs will be almost E-R-like.

\paragraph{Applicability of Theorem 4.2 on Graph Transformer.} We prove Theorem 4.2 on a simplified link predictor in  Appendix \ref{thmproofs} while in practice DiGress uses a complex Graph Transformer with its edge predictions modulated by FiLM layers \citep{vignac2023digressdiscretedenoisingdiffusion}. We argue this does not change the conclusion as this setup provides with a tractable model to study behavior. Further, the edge modulation in the DiGress transformer is applied per pair: each pair's modulation is computed from its own edge feature and from the global (timestep) vector, so it does not introduce dependence between different pairs. With constant node features, the predicted edge logits depend on each pair only through its own current state and a global summary of the graph, leading to E-R-like graphs and preventing node heterogeneity from emerging. Empirically, this is what we observe on the model used in practice: Fig.~\ref{fig:propexpts} shows the node propensity variance of the actual DiGress denoiser, with edge modulation, collapsing toward zero at high noise, while the LSP-based model retains heterogeneity throughout.

\paragraph{Structural node features.} The above issue has often been interpreted as a lack of ``expressivity'' of GNNs, hampering their denoising abilities. To fix this, many works add structural node features computed {\emph{a posteriori}} at each denoising step, such as degrees, cycle counts, and so on. While such node features indeed counteract the incapacity of GNNs to compute these quantities, they only offer limited fix to the E-R problem, as shown in the following Lemma proved in Appendix \ref{app:proofstatsconstant}.

\begin{lemma} \label{statsconstantlemma}
    Consider a dense E-R random graph with parameter $p = O(1)$ for simplicity. Consider a template graph $H$ of size $r$, and denote by $C_i$ the number of subgraphs equal to $H$ containing node $i$. Then $C = \mathbb{E}(C_i)$ does not depend on $i$, and with probability at least $1-\delta$,
    $$
    \std(C_i/C) \lesssim \sqrt{\log(n/\delta)/n}
    $$
\end{lemma}
Hence, in the limit, this does not solve the problem, as the standard deviation of such features also vanishes. This is shown in Figure~\ref{fig:lemma4.3} in the Appendix.

\section{Latent variable guided discrete diffusion}
\label{sec:jointdiffusion}
Lemma \ref{lem:ER} establishes that the E-R prior is not merely a sub-optimal
choice but an unavoidable consequence of edge independence and exchangeability, and Theorem \ref{thm:concentration_ER} and Lemma~\ref{statsconstantlemma} state that E-R graphs become difficult to denoise in the large graph limit, even with
structural features computed \emph{a posteriori}. 
To resolve this, we propose a \emph{new prior} that induces node-level heterogeneity directly through latent variables, along with the corresponding \emph{joint} diffusion process on latent variables and discrete graph structure. 

\begin{figure}[h]
    \centering
    \includegraphics[width=0.6\linewidth]{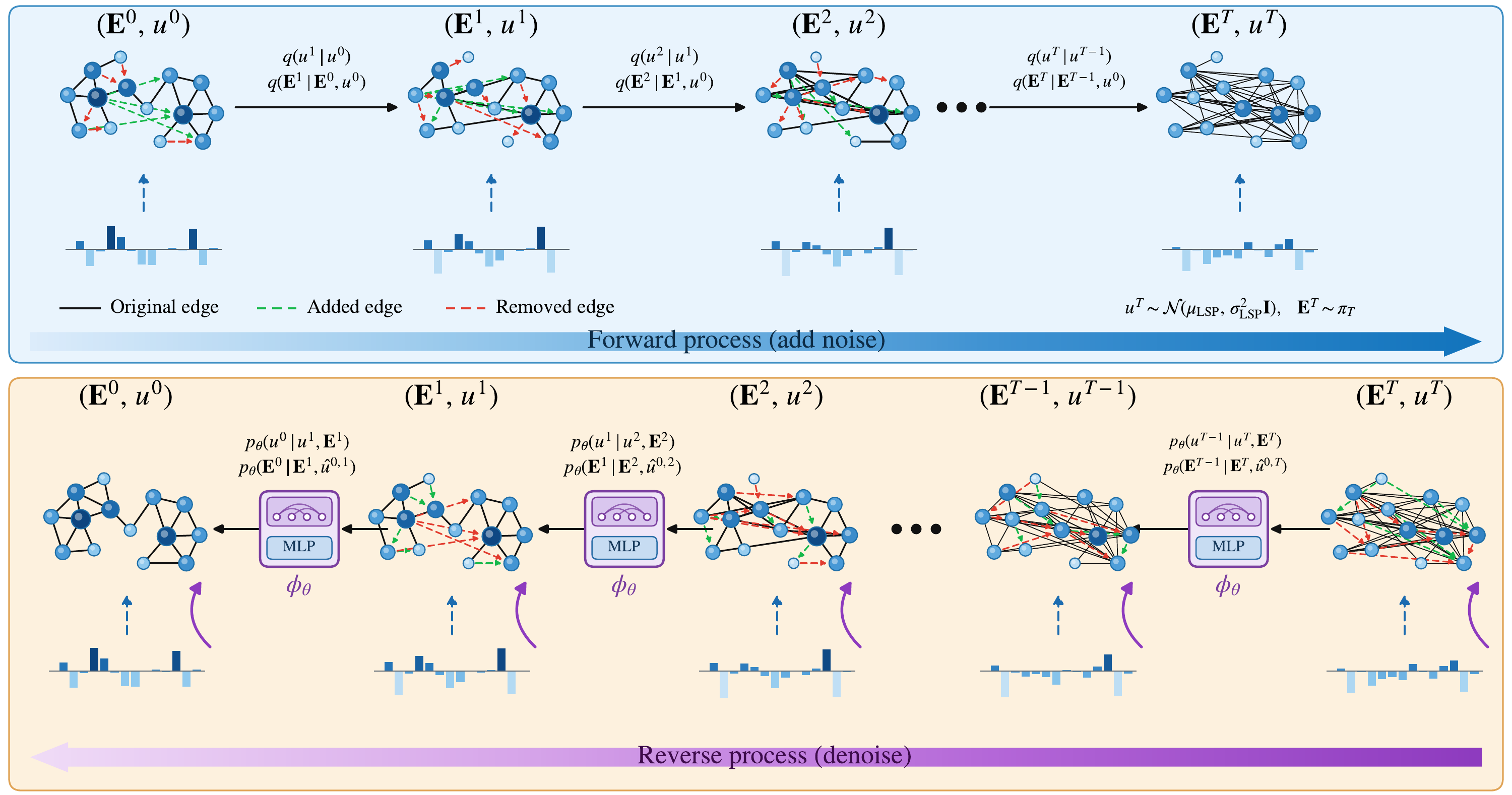}
\caption{\textbf{Joint diffusion over edges and latent sociability.}
\emph{Top:} the latent sociability $u^t$ and edge states $\mathbf{E}^t$ are jointly noised, with edge transitions conditioned on $u^0$, until reaching $u^T \sim \mathcal{N}(\mu_{\mathrm{LSP}},\sigma_{\mathrm{LSP}}^2)$ and $\mathbf{E}^T \sim \pi_T$.
\emph{Bottom:} during denoising, the graph transformer $\phi_\theta$ predicts the less noisy latent $\hat{u}^{t-1}$, which conditions the reverse edge transition $p_\theta(\mathbf{E}^{t-1}\mid \mathbf{E}^t,\hat{u}^{t-1})$.}
\label{fig:joint_diffusion}
    \label{fig:enter-label}
\end{figure}

\subsection{Latent Sociability Prior (LSP)}
Unlike existing latent graph diffusion approaches \cite{siraudin2026principled}, which encode the entire graph structure into high-dimensional latent space and may require complex autoencoders, the role of the latent variables here is
strictly to enforce node heterogeneity along the discrete diffusion process. This allows our choice of latent model to be extremely simple. 
\begin{definition}[Latent Sociability Distribution \citep{Norros_Reittu_2006}]
\label{def:lsp}
Let \(w_1,\dots,w_n \overset{i.i.d.}{\sim} p_W\) be latent sociability weights, where \(p_W\) is a
distribution on \(\mathbb{R}_{>0}\) to be specified in Section~\ref{jointdiffusionwithlsp}. The latent
sociability distribution defines a random undirected graph \(G\) with \emph{conditionally} independent Bernoulli edges as 
\begin{equation}\label{lspequation}
    \mathbb{P}(e_{ij}=1\mid w_i,w_j) = 1-e^{-w_i w_j}, \qquad i<j,
\end{equation}

\end{definition}

We refer to this distribution of random graphs as Latent Sociability Prior (LSP). By construction, the edges are conditionally independent given $w_1,\dots,w_n$ but marginally dependent, while global exchangeability is preserved. 
Intuitively, a node with high weights $w_i$ will tend to be more ``connected'': all the $p_{ij}$ will be higher.
Hence, a distribution $p_W$ with a large variance \emph{guarantees} a diverse profile of connection propensities (see Figure~\ref{fig:propexpts}), and empirically a variety of higher-order structures as well (Figure ~\ref{fig:higherorderstatscounts}), unlike the E-R prior. Integrating this latent-variable prior
into the discrete denoising framework of Section~\ref{ddframework} is not immediately obvious. We develop below a \emph{joint}
diffusion process that combines continuous Gaussian diffusion \citep{ddpm} on the latent weights
with the discrete corruption process on the graph edges.

\subsection{Joint diffusion with LSP} \label{jointdiffusionwithlsp}
A direct diffusion on the weights $w_i$ is not suitable, since Gaussian diffusion would not
preserve the positivity constraint and could produce invalid intermediate states. We therefore
reparameterize the weights as $w_i = \exp(u_i)$ with $u_i \in \mathbb{R}$ unconstrained. Under
a Gaussian prior on $u_i$, the induced prior on $w_i$ is log-normal. After this
reparameterization, the edge probability in Equation \ref{lspequation} becomes
\begin{equation}
p_{ij}=1-e^{-e^{u_i+u_j}}
\label{eq:lsp_link_u}
\end{equation}

\paragraph{Forward process.}
We define a joint forward noising process over the continuous latent vector $\mathbf{u}$ and the
discrete edge tensor $\mathbf{E}$. Given a clean graph $G^0$, one must first construct the
corresponding clean latent vector $\mathbf{u}^0$. While usual latent approaches require
training an encoder, in the LSP $\mathbf{u}^0$
admits a simple, approximate \emph{closed-form} expression.
\begin{lemma}\label{lem:u}
    Consider an LSP random graph. Assuming a sparse regime where $u_i = O(-c_n)$ for some
    constant $c_n \gg 0$, we have for all $i$
    $$
    \mathbb{E}(d_i \mid \mathbf{u}) = C e^{u_i} + O(ne^{-4c_n})
    $$
    for some constant $C = O(ne^{-c_n})$, where $d_i$ denotes the degree of node $i$.
\end{lemma}
Following this, we set $u_i^0 \approx \log(d_i + \delta) + \mu$ where $\delta>0$ is a small
smoothing constant. To better capture relative degree heterogeneity across different training
graphs, we remove the graph-level mean and obtain
\begin{equation}
u_i^0
=
\mu_{\mathrm{LSP}}
+
\log(d_i+\delta)
-
\frac{1}{|V|}\sum_{k\in V}\log(d_k+\delta)
\label{eq:u0_from_degree}
\end{equation}
where $\mu_{\mathrm{LSP}}$ is the mean of the terminal Gaussian prior. Similarly to enforcing the
marginal edge density in the E-R prior, we calibrate the parameters
$(\mu_{\mathrm{LSP}}, \sigma_{\mathrm{LSP}})$ from the training data, as detailed in
Appendix~\ref{constructlsp}.

We corrupt the latent variables using the shifted and scaled DDPM forward process

\begin{equation}
\mathbf{u}^t = \mu_{\mathrm{LSP}} +
\sqrt{\bar{\alpha}_t}\,(\mathbf{u}^0-\mu_{\mathrm{LSP}}) +
\sigma_{\mathrm{LSP}}\sqrt{1-\bar{\alpha}_t}\,\epsilon,
\qquad
\epsilon\sim\mathcal{N}(0,I)
\label{eq:forward_u_sample}
\end{equation}

so that $\mathbf{u}^t = \mathbf{u}^0$ at $t=0$ and $\mathbf{u}^t$ converges in distribution to
$\mathcal{N}(\mu_{\mathrm{LSP}}, \sigma_{\mathrm{LSP}}^2 I)$ as $t\to T$. The derivation and
reverse posterior are given in Appendix~\ref{appendix:ddpm_shifted}.

For the discrete component, we modify the transition matrix of Section~\ref{ddframework} by
replacing the global stationary $\mathbf{m}_E$ with a \emph{pair-specific} stationary

$\mathbf{p}_{ij}^0 = [\,1-p_{ij}^0,\;p_{ij}^0\,] \in \mathbb{R}^b$, where
$p_{ij}^0 = 1 - e^{-e^{u_i^0 + u_j^0}}$. This gives the pair-specific transition matrices

\begin{equation}
\mathbf{Q}_E^{t,ij}
=
\alpha_t I + \beta_t \mathbf{1}_b (\mathbf{p}_{ij}^0)^\top,
\qquad
\bar{\mathbf{Q}}_E^{t,ij}
=
\bar{\alpha}_t I + (1-\bar{\alpha}_t)\mathbf{1}_b(\mathbf{p}_{ij}^0)^\top
\label{eq:pairwise_transition_matrices}
\end{equation}

with entries $[\mathbf{Q}_E^{t,ij}]_{rs} = q(e_{ij}^t=s\mid e_{ij}^{t-1}=r,\mathbf{u}^0)$.
Since $\mathbf{u}^0$ is a deterministic function of $G^0$, all matrices are fixed given
$(\mathbf{E}^0, \mathbf{u}^0)$, and the joint forward marginal factors over independent edges:

\begin{equation}
q(\mathbf{E}^t \mid \mathbf{E}^0, \mathbf{u}^0)
=
\prod_{1 \leq i,j \leq n} q(e_{ij}^t \mid e_{ij}^0, \mathbf{u}^0),
\qquad
q(e_{ij}^t \mid e_{ij}^0, \mathbf{u}^0) = e_{ij}^0\,\bar{\mathbf{Q}}_E^{t,ij}
\label{eq:q_Et_joint}
\end{equation}

This has the same closed-form structure as Equation \ref{eq:digress_forward_compact}, with
pair-specific stationaries, so the forward posteriors needed for reverse sampling remain
tractable \citep{vignac2023digressdiscretedenoisingdiffusion}. Crucially, parameterizing the
transitions by $\mathbf{u}^0$ rather than $\mathbf{u}^t$ is what makes this possible: if the
transitions depended on $\mathbf{u}^t$, the matrices would be random variables and the product
structure in Equation \ref{eq:q_Et_joint} would no longer hold.

\paragraph{Reverse process.}
Starting from $(\mathbf{u}^t, \mathbf{E}^t)$, the network $\phi_\theta$ predicts the clean
latent vector $\hat{\mathbf{u}}^{0,t} = \phi_\theta(\mathbf{u}^t, \mathbf{E}^t, t)$ and updates
both components separately. The reverse kernel factorizes as

\begin{equation}
p_\theta(\mathbf{u}^{t-1}, \mathbf{E}^{t-1} \mid \mathbf{u}^t, \mathbf{E}^t)
=
p_\theta(\mathbf{u}^{t-1} \mid \mathbf{u}^t, \mathbf{E}^t)\,
p_\theta(\mathbf{E}^{t-1} \mid \hat{\mathbf{u}}^{0,t}, \mathbf{E}^t)
\label{eq:joint_reverse}
\end{equation}

For the continuous component, the reverse update follows the DDPM posterior in the
shifted-and-scaled parameterization derived in Appendix~\ref{appendix:ddpm_shifted}. For the
discrete component, $\hat{\mathbf{u}}^{0,t}$ determines pair-specific reverse transition matrices

\begin{equation}
\hat{\mathbf{Q}}_E^{t,ij}
=
\alpha_t I + \beta_t \mathbf{1}_b (\hat{\mathbf{p}}_{ij}^{0,t})^\top,
\qquad
\hat{\bar{\mathbf{Q}}}_E^{t,ij}
=
\bar{\alpha}_t I + (1-\bar{\alpha}_t)\mathbf{1}_b(\hat{\mathbf{p}}_{ij}^{0,t})^\top
\label{eq:reverse_Q_matrices}
\end{equation}

where $\hat{p}_{ij}^{0,t} = 1 - e^{-e^{\hat{u}_i^{0,t}+\hat{u}_j^{0,t}}}$ and
$\hat{\mathbf{p}}_{ij}^{0,t} = [\,1-\hat{p}_{ij}^{0,t},\;\hat{p}_{ij}^{0,t}\,]$. The reverse
edge distribution factorizes as in Equation \ref{eq:reverse_factorization}, with each factor given by
\begin{equation}
p_\theta(e_{ij}^{t-1} \mid e_{ij}^t, \hat{\mathbf{u}}^{0,t})
=
\sum_{r\,\in\,\mathcal{B}}
q(e_{ij}^{t-1} \mid e_{ij}^t, e_{ij}=r;\,\hat{\mathbf{u}}^{0,t})\,
\hat{\mathbf{p}}_{ij}^{E,t}(r)
\label{eq:reverse_edge_update}
\end{equation}

where the posterior is computed from the pair-specific matrices in Equation \ref{eq:reverse_Q_matrices}.

\paragraph{Training objective.}
We retain the cross-entropy objective of Equation \ref{eq:digress_ce_edges} for edge prediction and
augment it with an MSE loss on the predicted clean latents:
\begin{equation}
\mathcal{L}
=
\sum_{1 \leq i,j \leq n} \mathrm{CE}(e_{ij},\,\hat{\mathbf{p}}_{ij}^{E,t})
+
\lambda_u \left(\frac{1}{n}\sum_{i=1}^n \|\hat{u}_i^{0,t} - u_i^0\|_2^2\right)
\label{eq:joint_loss}
\end{equation}
where $\lambda_u \in \mathbb{R}_+$ controls the relative weight of the latent regression term. The variational derivation
and its scope are detailed in Appendix~\ref{app:elbo}.

\section{Experiments}
\paragraph{Datasets.} We introduce a new synthetic benchmark by sampling subgraphs from two real-world datasets, Cora \citep{Cora} and Chameleon \citep{wikipediadatasets}, which we refer to as CoraChameleon. Unlike commonly used synthetic benchmarks such as SBM or planar graphs, this better reflects the structural properties of real-world graphs; details are given in Appendix~\ref{datasetgen}. We generate $3000$ graphs for $n \in \{10,20,50,100\}$ and $1400$ for $n=500$, across node sizes $n \in \{10,20,50,100,500\}$. We also conduct experiments on the commonly used benchmarks Ego and Community datasets \citep{graphrnn}, with a $60/20/20$ train/test/validation split.
\subsection{JointLSP preserves node heterogeneity at high noise} \label{nodepropsection}
To study whether the proposed joint diffusion framework with LSP preserves node-level
heterogeneity during the diffusion process, we define a node propensity score $\mathcal{S}_i^t$
using the predicted clean edge probabilities at diffusion step $t$. For each node $i$, we compute $\mathcal{S}^t_i = \frac{1}{n-1}\sum_{j \neq i} \hat{\mathbf{p}}_{ij}^{E,t}$, where $\hat{\mathbf{p}}_{ij}^{E,t}$ denotes the
model's predicted clean edge probability between nodes $i$ and $j$ at time $t$. Intuitively, $\mathcal{S}_i^t$ measures the average connectivity propensity of node $i$ at a given noise
level. To quantify how much this varies across nodes, we compute
${V^t} =
{\mathrm{Var}_{i=1,\dots,n}(\mathcal{S}_i^t)}$. 

We evaluate this metric at different diffusion steps $t \in \{500, 50, 0\}$, where $t=500$ corresponds to the maximally noisy state, over training graphs in a batch. For each graph, $V^t$
is computed across all nodes, and the mean and standard deviation are reported over graphs in
the batch. This is shown for CoraChameleon-Node100 in Figure \ref{fig:propexpts}. The results show that the node heterogeneity induced by the JointLSP is preserved even at high noise levels. In contrast, under the original DiGress marginal prior without node features, the predicted edge probabilities become increasingly uniform across node pairs as $t \to T$, causing the node propensity scores to collapse toward a common value. This indicates that, at highly noisy states, the model effectively starts from an ER graph and must rely on auxiliary handcrafted features to recover richer structure during denoising. By contrast, our framework allows the model to retain variation in node propensity throughout the reverse process, providing an intrinsic source of heterogeneity without requiring such external features.

\begin{figure}[h]
    \centering
    \begin{subfigure}{0.46\linewidth}
        \centering
        \includegraphics[width=\linewidth]{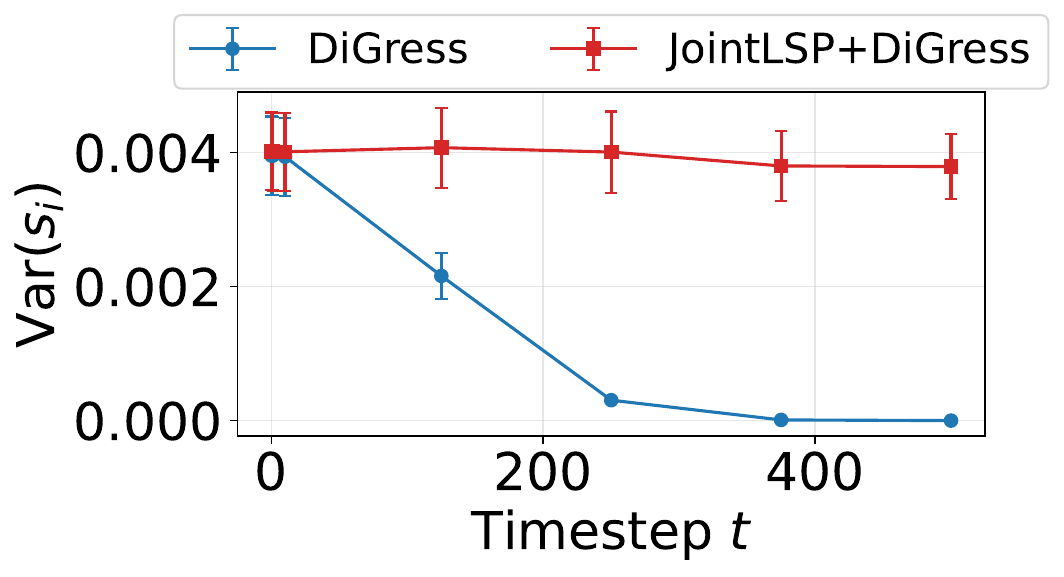}
        \caption{Forward Noising Process.}
          \label{fig:vnormplot}
    \end{subfigure}
    \begin{subfigure}{0.46\linewidth}
        \centering
        \includegraphics[width=\linewidth]{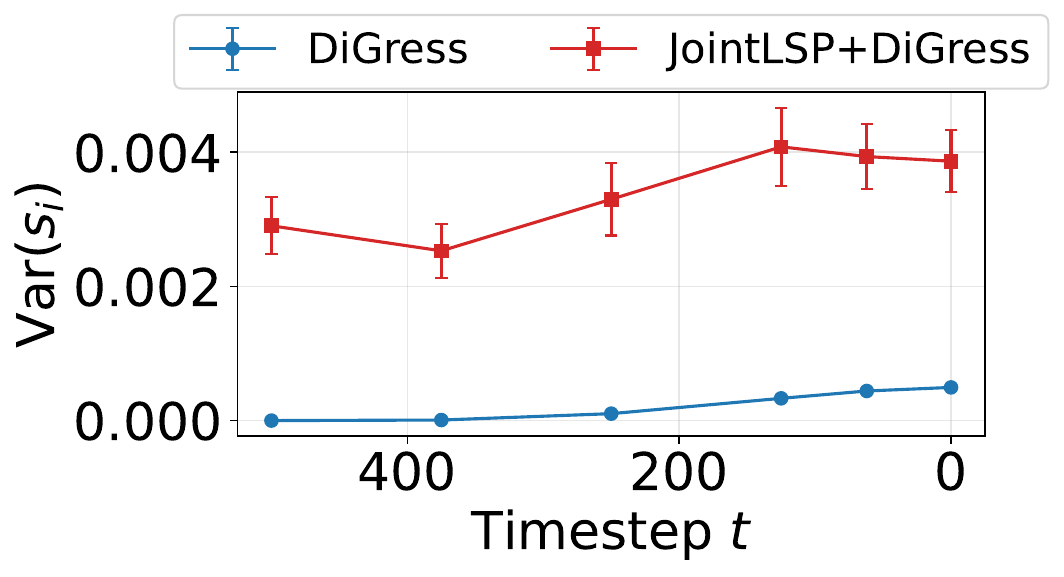}
        \caption{Reverse Denoising Process.}
        \label{fig:vnormreverse}
    \end{subfigure}
    \caption{We compare the variance of node propensity scores for both the original DiGress and JointLSP+DiGress (LSP in the figure). We observe at terminal noisy state, our prior helps maintain a healthy level of node heterogeneity as opposed to DiGress which loses node distinguishability during the diffusion process.}
    \label{fig:propexpts}
\end{figure}

\subsection{Experiments on graph generation} \label{experimentsgraphgeneration}
\paragraph{Evaluation metrics.}
We evaluate generated graphs using degree distribution, clustering coefficient, triangle count, and
degree assortativity. We additionally report the power-law exponent of the degree sequence (PLE)
and characteristic path length (CPL), following \citep{edgedependency} and
\citep{chen2023efficientdegreeguidedgraphgeneration}. For each statistic, we report the MMD ratio $r = \frac{\mathrm{MMD}(\mathrm{generated},\,\mathrm{test})^2}
         {\mathrm{MMD}(\mathrm{train},\,\mathrm{test})^2}$ following \citep{vignac2023digressdiscretedenoisingdiffusion}. Lower is better, with
$r=1$ as the natural baseline.

\paragraph{Ablation with auxiliary features.}
We first establish a DiGress \citep{vignac2023digressdiscretedenoisingdiffusion} baseline using the full set of recommended auxiliary features. We find that the performance of existing graph
diffusion models relies heavily on these additional node and graph-level features rather than on the diffusion process
alone. These features fall into two categories: \emph{structural features} such as $k$-cycle counts, which we address in Lemma \ref{statsconstantlemma}, and \emph{spectral features} such as eigenvectors and eigenvalues computed at every denoising step. We note however that the latter are expensive to compute, and prohibitive at scale. The full list is given in Appendix
Table~\ref{tab:digress_feature_names}. In Table~\ref{tab:digress_feature_ablation_summary} we present results on CoraChameleon-Node100 for three methods: DiGress+$k$ structural features,
by randomly choosing feature subsets of size $k$, averaged over 3 runs, JointLSP+DiGress, which augments DiGress with our joint diffusion framework and LSP but uses \emph{no}
auxiliary features ($k=0$), and finally Digress with the full set of structural and spectral features ($k=17$), which is more computationally intensive. The training details
and hyperparameters are given in Appendix~\ref{trainingdetails}. 

We can observe that structural features are consistently insufficient: even with
$k=10$ features, performance on higher-order statistics such as triangle count and assortativity
remains poor. We also conduct another ablation study in Table \ref{tab:digress_degree}, where we supply degree as both node and graph-level feature and observe similar results. This is consistent with Lemma~\ref{statsconstantlemma}: structural features
computed on a near-ER graph are themselves nearly constant across nodes, and therefore provide
little additional information to the denoiser at high noise levels. JointLSP+DiGress outperforms DiGress with up to $k=10$ features on clustering coefficient, triangle count, and assortativity,
demonstrating that improving the prior directly addresses these limitations. 
Nevertheless, JointLSP+DiGress is still outperformed by the use of \emph{spectral features}, even if much more expensive to compute. Integrating such spectral components into our novel joint diffusion process is a major avenue for future work.

\paragraph{Unconditional graph generation without node features.}
In Table~\ref{tab:mmdtablecombined} we present results for CoraChameleon datasets with
$n \in \{10, 20, 50, 100\}$ comparing DiGress without features against our proposed method.
Our goal is not to claim state-of-the-art performance, but to isolate the no-node-feature
setting in which the standard discrete diffusion pipeline is most constrained. Our method
shows consistent improvements on triangle count and assortativity
\citep{Newman_2003,doi:10.1073/pnas.0912671107}. At $n=100$, DiGress performs better on degree
distribution and PLE, which likely reflects occasional over-correction of the degree distribution under the stronger node heterogeneity induced by LSP. Additional results are reported in the appendix: results for JointLSP+SparseDiff on CoraChameleon across $n \in \{10,20,50,100,500\}$
and on the Ego and Community datasets are given in Appendix \ref{additionalresults}. Finally,
Table ~\ref{tab:digress_degree} examines whether supplying node degree directly as an
auxiliary feature to DiGress can close the gap, showing that while degree features improve some
metrics, they cannot recover the higher-order structural statistics that the LSP improves
confirming that the benefit of our method is not reducible to the degree information used in
constructing $\mathbf{u}^0$. 

\begin{table}[h]
\centering
\caption{Feature subset ablation for DiGress on CoraChameleon-Node100. JointLSP+DiGress ($k=0$, no features) is included for direct comparison. Bold: best result excluding the $k=17$ ceiling.}
\label{tab:digress_feature_ablation_summary}
\resizebox{\linewidth}{!}{%
\begin{tabular}{@{}lcccccc@{}}
\toprule
 & Degree $\downarrow$ & Clustering $\downarrow$ & Triangles $\downarrow$
 & Assortativity $\downarrow$ & PLE $\downarrow$ & CPL $\downarrow$ \\
\midrule
\multicolumn{7}{l}{\textit{DiGress + Structural features}} \\
\quad $k=2$  & $9.32_{\pm 3.08}$   & $5.08_{\pm 2.95}$  & $11.75_{\pm 4.50}$ & $118.18_{\pm 2.68}$  & $7.43_{\pm 2.65}$  & $47.23_{\pm 8.47}$  \\
\quad $k=4$  & $7.01_{\pm 0.53}$   & $9.73_{\pm 7.67}$  & $11.46_{\pm 4.55}$ & $107.68_{\pm 12.37}$ & $10.07_{\pm 5.75}$ & $43.99_{\pm 4.43}$  \\
\quad $k=8$  & $10.88_{\pm 11.42}$ & $5.89_{\pm 4.60}$  & $7.45_{\pm 5.08}$  & $79.81_{\pm 52.72}$  & $7.04_{\pm 5.38}$  & $36.55_{\pm 19.28}$ \\
\quad $k=10$ & $\mathbf{4.25}_{\pm 1.51}$   & $10.72_{\pm 6.46}$ & $10.06_{\pm 6.11}$ & $70.54_{\pm 40.85}$  & $\mathbf{6.99}_{\pm 3.69}$  & $\mathbf{34.47}_{\pm 21.49}$ \\
\midrule
\multicolumn{7}{l}{\textit{JointLSP+DiGress (no features)}} \\
\quad $k=0$ & $10.50_{\pm 0.16}$ & $\mathbf{2.07}_{\pm 0.10}$ & $\mathbf{2.33}_{\pm 0.15}$ & $\mathbf{59.72}_{\pm 0.75}$ & $10.79_{\pm 0.19}$ & $44.99_{\pm 0.29}$ \\
\midrule
\multicolumn{7}{l}{\textit{DiGress + Structural features + Spectral features}} \\
\quad $k=17$ & $0.93_{\pm 0.00}$ & $1.95_{\pm 0.00}$ & $1.68_{\pm 0.00}$ & $3.55_{\pm 0.00}$ & $1.53_{\pm 0.00}$ & $1.92_{\pm 0.00}$ \\
\bottomrule
\end{tabular}%
}
\end{table}

\begin{table}[h]
\centering
\caption{MMD ratios for unconditional graph generation under no-node features.}
\label{tab:mmdtablecombined}
\resizebox{\linewidth}{!}{%
\begin{tabular}{@{}lcccccc@{}}
\toprule
 & Degree $\downarrow$ & Clustering $\downarrow$ & Triangles $\downarrow$
 & Assortativity $\downarrow$ & PLE $\downarrow$ & CPL $\downarrow$ \\
\midrule
\multicolumn{7}{l}{\textit{DiGress (no features)}} \\
\quad $n=10$  & $6.30_{\pm 1.57}$   & $2.46_{\pm 0.34}$  & $15.28_{\pm 7.27}$  & $12.41_{\pm 1.55}$   & $35.40_{\pm 6.41}$          & $12.97_{\pm 4.19}$          \\
\quad $n=20$  & $22.47_{\pm 1.63}$  & $1.44_{\pm 0.00}$  & $58.45_{\pm 9.88}$  & $49.80_{\pm 3.08}$   & $\mathbf{15.84}_{\pm 2.39}$ & $11.46_{\pm 2.22}$          \\
\quad $n=50$  & $21.09_{\pm 0.91}$  & $1.51_{\pm 0.01}$  & $16.90_{\pm 0.23}$  & $70.67_{\pm 1.95}$   & $18.57_{\pm 1.07}$          & $55.49_{\pm 3.01}$          \\
\quad $n=100$ & $\mathbf{4.53}_{\pm 0.11}$   & $11.33_{\pm 0.27}$ & $14.82_{\pm 0.12}$  & $101.95_{\pm 2.13}$  & $\mathbf{4.64}_{\pm 0.17}$  & $\mathbf{39.68}_{\pm 1.14}$ \\
\midrule
\multicolumn{7}{l}{\textit{JointLSP+DiGress (no features)}} \\
\quad $n=10$  & $\mathbf{5.19}_{\pm 1.61}$  & $\mathbf{1.79}_{\pm 0.06}$ & $\mathbf{13.45}_{\pm 4.25}$ & $\mathbf{1.25}_{\pm 0.25}$  & $\mathbf{23.58}_{\pm 4.95}$ & $\mathbf{12.93}_{\pm 3.90}$ \\
\quad $n=20$  & $\mathbf{11.13}_{\pm 1.29}$ & $\mathbf{1.40}_{\pm 0.00}$ & $\mathbf{8.38}_{\pm 0.92}$  & $\mathbf{14.04}_{\pm 3.72}$ & $34.19_{\pm 3.47}$          & $\mathbf{7.82}_{\pm 1.24}$  \\
\quad $n=50$  & $\mathbf{8.01}_{\pm 0.42}$  & $\mathbf{1.40}_{\pm 0.00}$ & $\mathbf{4.87}_{\pm 0.80}$  & $\mathbf{33.18}_{\pm 1.16}$ & $\mathbf{10.77}_{\pm 0.69}$ & $\mathbf{30.65}_{\pm 1.71}$ \\
\quad $n=100$ & $10.50_{\pm 0.16}$          & $\mathbf{2.07}_{\pm 0.10}$ & $\mathbf{2.33}_{\pm 0.15}$  & $\mathbf{59.72}_{\pm 0.75}$ & $10.79_{\pm 0.19}$          & $44.99_{\pm 0.29}$          \\
\bottomrule
\end{tabular}%
}
\end{table}
\section{Conclusion} \label{conclusion}
We identify a fundamental limitation in discrete graph diffusion models for unconditional graph
generation without node features: the combination of independent edge updates and an E-R terminal
prior leads to a fixed-point pathology in the large graph limit, where any GNN denoiser applied
to an E-R graph approximately outputs an E-R graph. We propose the Latent Sociability Prior, a
simple latent-variable prior that is co-diffused with discrete edge updates in a joint diffusion
framework. By preserving node heterogeneity throughout the diffusion trajectory, LSP enables the
model to recover higher-order graph structure without relying on auxiliary structural features. We
validate this on graph generation benchmarks across multiple graph sizes.
\paragraph{Limitations and future work.}
This work focuses on a key limitation in existing graph diffusion models and several questions remain. Our framework is most suited to graphs with absent or limited node features. For molecular generation, where rich semantic features are available, the fixed-point pathology is less of a concern and the benefit of LSP is expected to be minimal. Combining LSP with auxiliary node features is a natural next step, but is non-trivial since LSP is already informative at high noise levels contrary to E-R prior which is uninformative. Thus, there is no guarantee that the auxiliary features will compound the benefits of the LSP. Promising future directions could include designing auxiliary features that complement the LSP, rather than duplicating the degree heterogeneity already captured by the sociability variables. In particular, spectral features appear important for improving community-level structure, but explicitly computing eigenvectors or eigenvalues at every denoising step can be expensive. One possibility is to extend the LSP itself with a community-aware component, by replacing the sociability prior with a calibrated community-conditioned variant. This would make the terminal prior informative about degree heterogeneity and community structure.

\begin{ack}
This work was funded by the European Union ERC-2024-
STG-101163069 MALAGA.
\end{ack}

\newpage
\medskip
\bibliographystyle{plainnat}
\bibliography{main}

\newpage
\appendix

\section{Appendix}
\begin{algorithm}[h]
\caption{Calibrating the LSP terminal prior}
\label{alg:calibrate_lsp}
\begin{algorithmic}[1]
\Require Training dataset $\mathcal{D}_{\mathrm{data}}$, smoothing constant $\delta>0$,
         number of quadrature nodes $K$
\For{each graph $G\in\mathcal{D}_{\mathrm{data}}$ and each node $i\in V(G)$}
    \State Compute $\tilde{u}_i = \log(d_i+\delta) - \frac{1}{|V|}\sum_{k\in V}\log(d_k+\delta)$
\EndFor
\State $\sigma_{\mathrm{LSP}} \leftarrow \mathrm{std}\bigl(\{\tilde{u}_i : i\in V(G),\, G\in\mathcal{D}_{\mathrm{data}}\}\bigr)$
    \Comment{std over all nodes across all training graphs}
\State $\rho_{\mathrm{data}} \leftarrow \frac{1}{|\mathcal{D}_{\mathrm{data}}|}\sum_{G\in\mathcal{D}_{\mathrm{data}}} \frac{2|E_G|}{n_G(n_G-1)}$
    \Comment{mean empirical edge density over training graphs}
\State Initialize $\mu_{\mathrm{LSP}}^0 \leftarrow \tfrac{1}{2}(\log\rho_{\mathrm{data}}-\sigma_{\mathrm{LSP}}^2)$
    \Comment{sparse-regime approximation, see Appendix \ref{constructlsp}}
\State Solve $\rho(\mu_{\mathrm{LSP}}) = \rho_{\mathrm{data}}$ for $\mu_{\mathrm{LSP}}$ using $K$-node
       Gauss-Hermite quadrature to evaluate $\rho(\cdot)$, starting from $\mu_{\mathrm{LSP}}^0$
\State \Return $(\mu_{\mathrm{LSP}},\,\sigma_{\mathrm{LSP}})$
\end{algorithmic}
\end{algorithm}

\begin{algorithm}[h]
\caption{Training JointLSP}
\label{alg:train_jointlsp}
\begin{algorithmic}[1]
\Require A graph $G=(\mathbf{1},\mathbf{E}^0)$, calibrated prior parameters
         $(\mu_{\mathrm{LSP}},\sigma_{\mathrm{LSP}})$ from Algorithm \ref{alg:calibrate_lsp}
\State Sample $t \sim \mathcal{U}(\{1,\dots,T\})$
\State Compute $\mathbf{u}^0$ from $G$ using Equation \ref{eq:u0_from_degree}
\State Sample $\mathbf{u}^t \sim q(\mathbf{u}^t\mid \mathbf{u}^0)$ using Equation \ref{eq:forward_u_sample}
\State Compute $\mathbf{p}_{ij}^0$ for all pairs $(i,j)$ using $p_{ij}^0 = 1 - e^{-e^{u_i^0 + u_j^0}}$
\State Sample $\mathbf{E}^t \sim q(\mathbf{E}^t\mid \mathbf{E}^0,\mathbf{u}^0)$ using
       Equation \ref{eq:q_Et_joint}
\State $(\hat{\mathbf{u}}^{0,t},\hat{\mathbf{p}}^{E,t}) \leftarrow \phi_\theta(\mathbf{u}^t,\mathbf{E}^t,t)$
\State Optimizer step on
       $\mathcal{L}(\hat{\mathbf{u}}^{0,t},\hat{\mathbf{p}}^{E,t};\mathbf{u}^0,\mathbf{E}^0)$
       using Equation \ref{eq:joint_loss}
\end{algorithmic}
\end{algorithm}

\begin{algorithm}[h]
\caption{Sampling from JointLSP}
\label{alg:sample_jointlsp}
\begin{algorithmic}[1]
\Require Trained network $\phi_\theta$, calibrated prior parameters
         $(\mu_{\mathrm{LSP}},\sigma_{\mathrm{LSP}})$
\State Sample $n$ from the empirical node-count distribution of $\mathcal{D}_{\mathrm{data}}$
\State Sample $\mathbf{u}^T \sim \mathcal{N}(\mu_{\mathrm{LSP}},\sigma_{\mathrm{LSP}}^2 I)$
\State Compute $\mathbf{p}_{ij}^T = [1-p_{ij}^T,\, p_{ij}^T]$ where
       $p_{ij}^T = 1-e^{-e^{u_i^T+u_j^T}}$ for all pairs $(i,j)$
\State Sample $\mathbf{E}^T$: each edge $(i,j)$ independently from
       $\mathrm{Bernoulli}(p_{ij}^T)$
\For{$t=T$ \textbf{down to} $1$}
    \State $(\hat{\mathbf{u}}^{0,t},\hat{\mathbf{p}}^{E,t}) \leftarrow \phi_\theta(\mathbf{u}^t,\mathbf{E}^t,t)$
    \State Sample $\mathbf{u}^{t-1} \sim p_\theta(\mathbf{u}^{t-1}\mid \mathbf{u}^t,\mathbf{E}^t)$
           using Equation \ref{eq:reverse_u_sample_appendix}
    \State Sample $\mathbf{E}^{t-1} \sim p_\theta(\mathbf{E}^{t-1}\mid \hat{\mathbf{u}}^{0,t},\mathbf{E}^t)$
           using Equation \ref{eq:reverse_edge_update}
\EndFor
\State \Return $\mathbf{E}^0$
\end{algorithmic}
\end{algorithm}

\subsection{Constructing the Latent Sociability Prior}
\label{constructlsp}
We derive the calibration procedure for the terminal prior parameters
$(\mu_{\mathrm{LSP}}, \sigma_{\mathrm{LSP}})$, which are estimated from the training data and fixed
before training begins. The goal is to choose these parameters so that graphs sampled from the LSP
terminal prior match the empirical edge density and degree heterogeneity of $\mathcal{D}_{\mathrm{data}}$.

\paragraph{Reparameterization.}
From Section~\ref{jointdiffusionwithlsp}, the latent sociability weights satisfy $w_i = \exp(u_i)$,
so $u_i = \log w_i \in \mathbb{R}$. We place a Gaussian prior on the log-weights,
$u_i \overset{i.i.d.}{\sim} \mathcal{N}(\mu_{\mathrm{LSP}}, \sigma_{\mathrm{LSP}}^2)$,
which induces a log-normal distribution on $w_i$. With $c=1$, the edge probability becomes
\begin{equation}
p_{ij} = 1 - e^{-e^{u_i + u_j}}
\label{eq:app_pij}
\end{equation}

\paragraph{Estimating $\sigma_{\mathrm{LSP}}$.}
Since the true weights $w_i$ are unobserved, we construct a proxy from the degree sequence. For
each node $i$ in each training graph $G$, define the centered log-degree
\begin{equation}
\tilde{u}_i
=
\log(d_i+\delta) - \frac{1}{|V|}\sum_{k\in V}\log(d_k+\delta)
\label{eq:utilde}
\end{equation}
where $\delta>0$ is a small smoothing constant. Centering removes the graph-level mean so that
$\tilde{u}_i$ captures relative degree heterogeneity within each graph. We then set
\begin{equation}
\sigma_{\mathrm{LSP}}
=
\mathrm{std}\bigl(\{\tilde{u}_i : i\in V(G),\, G\in\mathcal{D}_{\mathrm{data}}\}\bigr)
\label{eq:sigma_lsp}
\end{equation}
where the standard deviation is taken over all nodes across all training graphs.

\paragraph{Estimating $\mu_{\mathrm{LSP}}$.}
We calibrate $\mu_{\mathrm{LSP}}$ by matching the expected edge probability under the LSP to the
empirical edge density of the training data. The empirical edge density, averaged over training
graphs, is
\begin{equation}
\rho_{\mathrm{data}}
=
\frac{1}{|\mathcal{D}_{\mathrm{data}}|}
\sum_{G\in\mathcal{D}_{\mathrm{data}}}
\frac{2|E_G|}{n_G(n_G-1)}
\label{eq:rho_data}
\end{equation}
Under the LSP with parameters $(\mu_{\mathrm{LSP}}, \sigma_{\mathrm{LSP}})$, the expected edge
probability between any two nodes is
\begin{equation}
\rho(\mu_{\mathrm{LSP}})
=
\mathbb{E}[p_{ij}]
=
\mathbb{E}\bigl[1-e^{-e^{u_i+u_j}}\bigr]
\label{eq:rho_mu}
\end{equation}

Note that $c$ enters Equation \ref{eq:rho_mu} only through the product $c\,e^{u_i+u_j}$, 
which in the sparse regime reduces to $c\,e^{2\mu_{\mathrm{LSP}}+\sigma_{\mathrm{LSP}}^2}$ 
(see Equation \ref{eq:mu_lsp_init}). Any value of $c>0$ can therefore be absorbed into 
$\mu_{\mathrm{LSP}}$ via the substitution $\mu_{\mathrm{LSP}} \leftarrow 
\mu_{\mathrm{LSP}} + \frac{1}{2}\log c$, leaving the model unchanged. We fix $c=1$ 
throughout without loss of generality, since $\mu_{\mathrm{LSP}}$ is calibrated 
directly to match $\rho_{\mathrm{data}}$.

Since $u_i, u_j \overset{i.i.d.}{\sim} \mathcal{N}(\mu_{\mathrm{LSP}}, \sigma_{\mathrm{LSP}}^2)$,
the sum $S = u_i + u_j \sim \mathcal{N}(2\mu_{\mathrm{LSP}}, 2\sigma_{\mathrm{LSP}}^2)$.
Writing $S = 2\mu_{\mathrm{LSP}} + \sqrt{2}\,\sigma_{\mathrm{LSP}}\,Z$ with $Z\sim\mathcal{N}(0,1)$,
Equation \ref{eq:rho_mu} becomes
\begin{equation}
\rho(\mu_{\mathrm{LSP}})
=
1
-
\mathbb{E}_{Z\sim\mathcal{N}(0,1)}\!\left[e^{-e^{2\mu_{\mathrm{LSP}}+\sqrt{2}\,\sigma_{\mathrm{LSP}}\,Z}}\right]
\label{eq:rho_mu2}
\end{equation}
We solve $\rho(\mu_{\mathrm{LSP}}) = \rho_{\mathrm{data}}$ for $\mu_{\mathrm{LSP}}$ using
root-finding with the expectation in Equation \ref{eq:rho_mu2} evaluated via $K$-node Gauss-Hermite
quadrature, which replaces the integral by a deterministic weighted sum.

\paragraph{Initialization of the root-finding.}
In the sparse regime where $e^{u_i+u_j}$ is small, $p_{ij} = 1 - e^{-e^{u_i+u_j}} \approx e^{u_i+u_j}$,
so
\begin{equation}
\rho_{\mathrm{data}}
\approx
\mathbb{E}\bigl[e^{u_i+u_j}\bigr]
=
e^{2\mu_{\mathrm{LSP}}+\sigma_{\mathrm{LSP}}^2}
\end{equation}
Solving for $\mu_{\mathrm{LSP}}$ gives the initialization
\begin{equation}
\mu_{\mathrm{LSP}}^0
=
\tfrac{1}{2}\bigl(\log\rho_{\mathrm{data}} - \sigma_{\mathrm{LSP}}^2\bigr)
\label{eq:mu_lsp_init}
\end{equation}
where we have used $c=1$. Since real-world graphs are typically sparse, this approximation is mild
and provides a close starting point for the root-finder in practice.

\subsection{Shifted and scaled DDPM for latent sociability variables}
\label{appendix:ddpm_shifted}

The following derivation is stated for a single coordinate $u_i$; the result extends to the full
vector $\mathbf{u}$ coordinatewise.

\paragraph{Forward process.}
The standard DDPM forward process \citep{ddpm} takes the form
\begin{equation}
x^t
=
\sqrt{\bar{\alpha}_t}\,x^0
+
\sqrt{1-\bar{\alpha}_t}\,\epsilon,
\qquad
\epsilon \sim \mathcal{N}(0,I).
\label{eq:ddpm_standard}
\end{equation}

In our setting the latent sociability variables should diffuse toward a Gaussian prior with mean
$\mu_{\mathrm{LSP}}$ and variance $\sigma_{\mathrm{LSP}}^2$, rather than toward $\mathcal{N}(0,I)$.

To obtain this, we apply the DDPM construction to the centered and scaled variable

$(u - \mu_{\mathrm{LSP}})/\sigma_{\mathrm{LSP}}$:
\begin{equation}
\frac{u^t-\mu_{\mathrm{LSP}}}{\sigma_{\mathrm{LSP}}}
=
\sqrt{\bar{\alpha}_t}\,
\frac{u^0-\mu_{\mathrm{LSP}}}{\sigma_{\mathrm{LSP}}}
+
\sqrt{1-\bar{\alpha}_t}\,\epsilon,
\qquad
\epsilon \sim \mathcal{N}(0,I).
\label{eq:ddpm_shift_scale}
\end{equation}

Multiplying through by $\sigma_{\mathrm{LSP}}$ and adding $\mu_{\mathrm{LSP}}$ gives

\begin{equation}
u^t
=
\mu_{\mathrm{LSP}}
+
\sqrt{\bar{\alpha}_t}(u^0-\mu_{\mathrm{LSP}})
+
\sigma_{\mathrm{LSP}}\sqrt{1-\bar{\alpha}_t}\,\epsilon,
\qquad
\epsilon \sim \mathcal{N}(0,I),
\label{eq:forward_u_appendix}
\end{equation}

so that $u^t = u^0$ at $t=0$ and $u^t \to \mathcal{N}
(\mu_{\mathrm{LSP}},\sigma_{\mathrm{LSP}}^2)$
as $t \to T$.

\paragraph{Reverse process.}
Since Equation \ref{eq:ddpm_shift_scale} is the standard DDPM forward process applied to
$\tilde{u}^t := (u^t - \mu_{\mathrm{LSP}})/\sigma_{\mathrm{LSP}}$, the reverse posterior in
$\tilde{u}$-space is the standard DDPM Gaussian posterior \citep{ddpm}:

\begin{equation}
q(\tilde{u}^{t-1}\mid \tilde{u}^t,\tilde{u}^0)
=
\mathcal{N}(\tilde{u}^{t-1}\mid \tilde{\mu}_{t-1},\,\tilde{\beta}_t)
\end{equation}
with mean
\begin{equation}
\tilde{\mu}_{t-1}
=
\frac{\sqrt{\bar{\alpha}_{t-1}}\,\beta_t}{1-\bar{\alpha}_t}\,\tilde{u}^0
+
\frac{\sqrt{\alpha_t}(1-\bar{\alpha}_{t-1})}{1-\bar{\alpha}_t}\,\tilde{u}^t
\label{eq:ddpm_reverse_centered_mean}
\end{equation}

and variance

\begin{equation}
\tilde{\beta}_t = \frac{\beta_t(1-\bar{\alpha}_{t-1})}{1-\bar{\alpha}_t}.
\label{eq:ddpm_reverse_centered_var}
\end{equation}

Substituting $\tilde{u}^t = (u^t - \mu_{\mathrm{LSP}})/\sigma_{\mathrm{LSP}}$ and
$\tilde{u}^0 = (u^0 - \mu_{\mathrm{LSP}})/\sigma_{\mathrm{LSP}}$ into
Equation \ref{eq:ddpm_reverse_centered_mean} and multiplying through by $\sigma_{\mathrm{LSP}}$ gives
the posterior mean in $u$-space:

\begin{equation}
\mu_{t-1}
=
\mu_{\mathrm{LSP}}
+
\frac{\sqrt{\bar{\alpha}_{t-1}}\,\beta_t}{1-\bar{\alpha}_t}(u^0-\mu_{\mathrm{LSP}})
+
\frac{\sqrt{\alpha_t}(1-\bar{\alpha}_{t-1})}{1-\bar{\alpha}_t}(u^t-\mu_{\mathrm{LSP}})
\label{eq:reverse_u_mean_appendix}
\end{equation}

The posterior variance scales by $\sigma_{\mathrm{LSP}}^2$, giving

\begin{equation}
\tilde{\sigma}_t^2
=
\sigma_{\mathrm{LSP}}^2\,\tilde{\beta}_t
=
\sigma_{\mathrm{LSP}}^2\,
\frac{\beta_t(1-\bar{\alpha}_{t-1})}{1-\bar{\alpha}_t}
\label{eq:reverse_u_var_appendix}
\end{equation}

so that

\begin{equation}
q(u^{t-1}\mid u^t,u^0) = \mathcal{N}(u^{t-1}\mid \mu_{t-1},\,\tilde{\sigma}_t^2).
\end{equation}

The reverse sampling step is therefore

\begin{equation}
u^{t-1}
=
\mu_{t-1}
+
\sigma_{\mathrm{LSP}}
\sqrt{\frac{\beta_t(1-\bar{\alpha}_{t-1})}{1-\bar{\alpha}_t}}\,z,
\qquad
z\sim\mathcal{N}(0,1),
\label{eq:reverse_u_sample_appendix}
\end{equation}

which is the update used in Section~\ref{jointdiffusionwithlsp},
with $u^0$ replaced by the network prediction $\hat{u}^{0,t}$ at each timestep.

\section{Proofs} \label{app:proof}

\subsection{Proof of Lemma \ref{lem:ER}}\label{app:prooflemER}

\begin{proof}[Proof of Lemma \ref{lem:ER}]
Since the edges are discrete in $\{0,1\}$, they are independent Bernoulli random variables. We note $\mathbb{P}(e_{ij} = 1) = p_{ij}$ its parameters. Denote by $A = [e_{ij}]_{ij}$ the adjacency matrix of a drawn graph, we have $P = [p_{ij}]_{ij} = \mathbb{E}(A)$. By exchangeability, $\sigma A\sigma^\top$ has the same distribution as $A$ for any permutation matrix $\sigma$, and therefore
\begin{equation}
\mathbb{E}[\sigma A\sigma^\top]=\mathbb{E}[A]
\end{equation}
which is
\begin{equation}
\pi P\pi^\top=P
\end{equation}
The only matrices invariant under conjugation by all permutation matrices have constant off-diagonal entries. Since the diagonal is zero for simple graphs, it follows that $P_{ij}=p$ for all $i\neq j$. Therefore the model is Erd\H{o}s--R\'enyi.
\end{proof}

\subsection{Proof of Thm.~\ref{thm:concentration_ER}}\label{thmproofs}

\paragraph{GNNs.} Denote by $\vX^{(0)} = \vX$ the input node features. The GNNs we are examining are all under the form
    $$
    \vX^{(k)} = \sigma\left( \vX^{(k-1)} \theta_0^{(k)} + S_k(\vX^{(k-1)}) \vX^{(k-1)} \theta_1^{(k)} + 1_n \vb^{(k)}\right) \in \mathbb{R}^{n \times d_k}
    $$
where $S_k(\vX^{(k-1)}) \in \mathbb{R}^{n \times n}$ are message-passing matrices that depends on the graph structure and may depend on the layer $k$ and $\vX^{(k-1)}$ for attention-based models, $\theta_i^{(k)} \in \mathbb{R}^{d_{k-1} \times d_k}$ are trainable weights, $\vb^{(k)} \in \mathbb{R}^{d_k}$ are trainable biases, and $\sigma$ is an activation function that we assume to be $1$-Lipschitz. All models can then be obtained by varying the choice of $S_k(\vX)$, which we denote by $S$ when it is fixed:
\begin{itemize}
    \item for GCN, $$S = \tilde D^{-1/2}\tilde A \tilde D^{-1/2}$$ where $\tilde A = Id + A$ and $\tilde D$ is the corresponding degree matrix.
    \item GIN uses sum-aggregation, which would correspond to $S=A$. However for random graphs this choice may diverge when $n$ increases if not properly normalized. We choose the classical normalization for random graphs \cite{kerivenPE} by the expected density of the graph: for the ER model considered in the result,
    $$S = \frac{1}{pn} A$$
    assuming that $p$ is known for simplicity. Note that, for a fixed graph size $n$, this normalization can be absorbed by the weight $\theta_1$, such that this is indeed equivalent to the vanilla GIN;
    \item for GAT, we consider attention coefficients $c_k(\vx, \vx')$ at layer $k$ (typically exponentials of dot products), and normalized attention:
    $$
        [S_k(\vX)]_{ij} = \frac{c_k(\vx_i, \vx_j)}{\sum_{\ell \sim i}c_k(\vx_i, \vx_\ell)}
    $$
    We will assume that $c_k(\cdot) \geq c_{\min} >0$, which is for instance the case for exponential attention coefficients on bounded features. Graph Transformers are essentially the same model, when stripped of various normalization layers (which we do not consider here for simplicity).
\end{itemize}

After $K$ layers, the node representations are used pairwise to obtain the edge prediction
$$
\hat p_{ij} = [\varphi(\vx_i^{(K)}, \vx_j^{(K)})]_{i,j=1}^n \in [0,1]
$$
where $\varphi$ is a link function that we assume to be $L_\varphi$-Lipschitz in each component.

\begin{proof}[Proof of Thm.~\ref{thm:concentration_ER}]
Recall that $\std(\hat p) = \sqrt{\frac{1}{n^2} \sum_{ij} (\hat p_{ij} - \bar p)^2}$ where $\bar p = \sum_{ij} \hat p_{ij}$. First observe that
$$
\sum_{ij} (\hat p_{ij} - \bar p)^2 = \min_p \sum_{ij} (p_{ij} - p)^2
$$
and similarly, for $\vX \in \mathbb{R}^{n \times d}$,
$$
\std(\vX) = \sqrt{\frac{1}{n}\sum_i \norm{\vx_i - \bar \vx}^2} = \min_\vx \sqrt{\frac{1}{n}\sum_i \norm{\vx_i - \vx}^2}
$$
where $\bar \vx = \frac{1}{n} \sum_i \vx_i$. We therefore denote $\bar \vx^{(k)} = \frac{1}{n} \sum_i \vx^{(k)}_i$.

Hence, for any $\vx$, by Lipschitz property of $\varphi$ and a triangular inequality
\begin{align*}
\std(\hat p) &\leq \sqrt{\frac{1}{n^2} \sum_{ij} (\varphi(\vx_i^{(K)}, \vx_j^{(K)}) - \varphi(\vx,\vx))^2} \\
&\leq \sqrt{\frac{1}{n^2} \sum_{ij} (\varphi(\vx_i^{(K)}, \vx_j^{(K)}) - \varphi(\vx_i^{(K)},\vx))^2} + \sqrt{\frac{1}{n^2} \sum_{ij} (\varphi(\vx_i^{(K)}, \vx) - \varphi(\vx,\vx))^2} \\
&\leq 2L_\varphi \sqrt{\frac{1}{n}\sum_i \norm{\vx_i^{(K)} - \vx}^2 }
\end{align*}
and by minimizing over $\vx$, we get $\std(\hat p) \leq 2L_\varphi \std(\vX^{(K)})$.

Now, denote by $E_k = \std(\vX^{(k)})$. Define $\normFn{\cdot} = \norm{\cdot}_F/\sqrt{n}$ for short, where $\norm{\cdot}_F$ is the Frobenius norm, and observe that $\std(\vX) = \normFn{\vX - 1_n \bar \vx^\top}$.
Moreover, observe that $\normFn{S\vX \theta} \leq \norm{S}_2 \normFn{\vX} \norm{\theta}_2$.

We have the recursion equation: since minimizing over $\vx$ is more general than over $\sigma(\vx)$, and using the fact that $\sigma$ is $1$-Lipschitz and we can decompose $\min_x \norm{y + y' - x} = \min_{x,x'} \norm{y + y' - (x+x')} \leq \min_x \norm{y-x} + \min_x \norm{y'-x}$,
\begin{align*}
    \std(\vX^{(k)}) &= \min_\vx \normFn{\vX^{(k)} - 1_n \vx^\top} \leq \min_\vx \normFn{\vX^{(k)} - 1_n \sigma(\vx)^\top} \\
    &= \min_\vx \normFn{\sigma\left( \vX^{(k-1)} \theta_0^{(k)} + S_k(\vX^{(k-1)}) \vX^{(k-1)} \theta_1^{(k)} + 1_n \vb^{(k)}\right) - \sigma(1_n \vx^\top)} \\
    &\leq \min_\vx \normFn{ \vX^{(k-1)} \theta_0^{(k)} + S_k(\vX^{(k-1)}) \vX^{(k-1)} \theta_1^{(k)} + 1_n \vb^{(k)} - 1_n \vx^\top} \\
    &\leq \min_\vx \normFn{\vX^{(k-1)} \theta_0^{(k)} - 1_n \vx^\top} + \min_\vx \normFn{S_k(\vX^{(k-1)}) \vX^{(k-1)} \theta_1^{(k)} - 1_n \vx^\top}
\end{align*}

For the first term, by picking $\vx = {\theta^{(k)}_0}^\top \bar \vx^{(k-1)}$, we get
\begin{align*}
\min_\vx \normFn{\vX^{(k-1)} \theta_0^{(k)} - 1_n \vx^\top} \leq  \normFn{(\vX^{(k-1)} - 1_n (\bar \vx^{(k-1)})^\top)\theta_0^{(k)}} \leq E_{k-1} \norm{\theta_0^{(k)}}_2
\end{align*}
We need to divide the second term, since the term $1_n \vx^\top$ cannot necessarily be written as $S_k(\vX^{(k-1)}) 1_n {\vx'}^\top$ for some $\vx'$. So we introduce a term $S_k(\vX^{(k-1)}) 1_n {\vx'}^\top$ for some $\vx'$:
\begin{align*}
    \min_\vx \normFn{S_k(\vX^{(k-1)}) \vX^{(k-1)} \theta_1^{(k)} - 1_n \vx^\top} &\leq \normFn{S_k(\vX^{(k-1)}) \left(\vX^{(k-1)} \theta_1^{(k)} - 1_n {\vx'}^\top\right)} \\&\quad+ \min_{\vx} \normFn{S_k(\vX^{(k-1)}) 1_n {\vx'}^\top - 1_n \vx^\top}
\end{align*}
Then, as previously, by picking $\vx' = \tilde \vx_k = {\theta^{(k)}_1}^\top \bar \vx^{(k-1)}$ the first term in this equation can be bounded by
\begin{align*}
    \normFn{S_k(\vX^{(k-1)}) \left(\vX^{(k-1)} \theta_1^{(k)} - 1_n {\vx'}^\top\right)} \leq \norm{S_k(\vX^{(k-1)})}_2 E_{k-1} \norm{\theta_1^{(k)}}_2
\end{align*}
Note that we must then bound $\norm{S_k(\vX^{(k-1)})}_2$, which is random and depend on the random graph and chosen GNN architecture.

We denote the remaining term by
$$
r_n^{(k)} = \min_{\vx} \normFn{S_k(\vX^{(k-1)}) 1_n {\tilde \vx_k}^\top - 1_n \vx^\top}
$$
which we will again bound for each example.

At the end of the day, we obtain
\begin{align*}
    E_k &\leq \underbrace{\left(\norm{\theta_0^{(k)}}_2 + \norm{S_k(\vX^{(k-1)})}_2 \norm{\theta_1^{(k)}}_2 \right)}_{u_k} E_{k-1} + r_n^{(k)} \\
    &= \underbrace{\sum_{\ell=1}^k \left[\prod_{p=\ell+1}^k u_p\right] r_n^{(\ell)}}_{r_n(\delta)} + \left[\prod_{\ell=1}^k u_p\right] E_0
\end{align*}
For all example, we are going to show that $u_k = O(1)$ with probability $1-\delta$ if $n$ is large enough, and bound $r_n(\delta)$ wrt $n$ and $\delta$. Our general strategy will be to write
\begin{equation}\label{eq:boundr}
    r_n^{(k)} = \min_{\vx} \normFn{S_k(\vX^{(k-1)}) 1_n {\tilde \vx_k}^\top - 1_n \vx^\top} \leq \norm{\tilde \vx_k} \min_c \max_i \left| \sum_{j \sim i} s_{ij} - c\right|
\end{equation}
for some constant $c \in \mathbb{R}$ (often $1$ in our examples), and $s_{ij}$ are the elements of $S_k(\vX^{(k-1)})$ for short. We remark that $\norm{\tilde \vx_k} \leq \norm{\theta_1^{(k)}}_2 \normFn{\vX^{(k-1)}}$ and we can use Lemma \ref{lem:bound} to bound this term, which again involves $u_k$.

Our main tool will be the following Lemma.

\begin{lemma}[Lemma 3 in \cite{sparseandsmooth}]\label{lem:degree}
    Consider an ER model with parameter $p$. Then, with probability at least $1-\delta$, the degrees $d_i$ satisfy
    \begin{equation}
        \max_i \left| \frac{d_i}{pn} - 1 \right| \lesssim \sqrt{\frac{\log(n/\delta)}{pn}}
    \end{equation}
\end{lemma}
In particular, if
\begin{equation}\label{eq:boundn}
    n \gtrsim \frac{\log(n/\delta)}{p}
\end{equation}
then with probability $1-\delta$, for all $i$
$$
\frac{pn}{2}\leq d_i \leq \frac{3pn}{2}
$$

We now bound $\norm{S_k(\vX^{(k-1)})}_2$ and $r_n$ through Equation \ref{eq:boundr} for our examples.

\paragraph{GAT.}
GAT is actually the simplest model to bound. Indeed, due to the normalization of the coefficients $\sum_{j \sim i} s_{ij} = 1$ at all layers, we directly obtain $r_n^{(k)} = 0$.

For the bound on $\norm{S_k(\vX)}_2$, we remark that under the bound Equation \ref{eq:boundn} on $n$, for all $i,j,k$
$$
\frac{c_k(\vx_i, \vx_j)}{\sum_{\ell \sim i} c_k(\vx_i, \vx_\ell)} \leq \frac{c_{\max}}{c_{\min}} \frac{2}{pn}
$$
and thus for any layer and $\vX$
\begin{align*}
\norm{S_k(\vX)}_2 &\leq \norm{S_k(\vX)}_F = \sqrt{\sum_{ij} s_{ij}^2} \lesssim \frac{c_{\max}}{c_{\min}} \sqrt{\sum_i \frac{d_i}{(pn)^2}} \lesssim \frac{c_{\max}}{p c_{\min}}
\end{align*}
hence the result $C = O(1/p^K)$.

\paragraph{GIN.}
For GIN, we have $S= A/(pn)$. Hence Equation \ref{eq:boundr} directly yields $r_n^{(k)}\leq \norm{\tilde \vx_k} \max_i \left| \frac{d_i}{pn} - 1 \right|$, and the last term is smaller than $\epsilon$ when $pn \geq \epsilon^{-2}\log(n/\delta)$ by Lemma \ref{lem:degree}.

For the bound on $\norm{S}_2$, we invoke \cite[Thm 5.2]{Lei2015}, such that for any $r>0$, there is a constant $C_r$ such that, if $n \geq C_r(1/\delta)^{r}$, then
$$
\norm{\frac{A}{pn} - \frac{1_{n \times n}}{n}}_2 \lesssim \frac{1}{\sqrt{pn}}
$$
and $\norm{\frac{1_{n \times n}}{n}}_2 = 1$ hence the result.

\paragraph{GCN.}
For GCN, if $pn \geq \epsilon^{-2}\log(n/\delta)$ by Lemma \ref{lem:degree} $|d_i - pn| = O(pn \epsilon)$ and thus since $\frac{1}{\sqrt{a}} = \frac{1}{\sqrt{b}} + O(\frac{|a-b|}{b^{3/2}})$,
$$
\frac{1}{\sqrt{d_i}} = \frac{1}{\sqrt{pn}} + O\left(\frac{\epsilon}{\sqrt{pn}}\right)
$$
From this we get
\begin{align*}
    \max_i \left| \sum_{j \sim i} s_{ij} - 1\right| &= \max_i \left| \frac{1}{\sqrt{d_i}}\sum_{j \sim i} \frac{1}{\sqrt{d_j}} - 1\right| \\
    &\leq \max_i \left|\frac{d_i}{pn} - 1\right| + O\left(\frac{\epsilon}{\sqrt{pn}}\right)
\end{align*}
and we go back to the GIN case.

For the bound on $\norm{S}_2$ we have directly that $\norm{S}_2\leq 1$, since its a symmetric normalized Laplacian (up to identity term).

\end{proof}

\begin{lemma}
    \label{lem:bound}
    \begin{equation}\label{eq:bound}
        \normFn{\vX^{(k)}} \leq \sum_{\ell=1}^k \left[\prod_{p=\ell+1}^k u_p\right] v_\ell + \left[\prod_{\ell=1}^k u_p\right] \normFn{\vX^{(0)}}
    \end{equation}
    where
    \begin{align*}
        u_k = \norm{\theta_0^{(k)}}_2 + \norm{S_k(\vX^{(k-1)})}_2 \norm{\theta_1^{(k)}}_2,\quad v_k = \norm{\vb^{(k)}}
    \end{align*}
\end{lemma}
\begin{proof}
    We write the recursion
    \begin{align*}
        \normFn{\vX^{(k)}} &= \normFn{\sigma\left( \vX^{(k-1)} \theta_0^{(k)} + S_k(\vX^{(k-1)}) \vX^{(k-1)} \theta_1^{(k)} + 1_n \vb^{(k)}\right)} \\
    &\leq \underbrace{\left(\norm{\theta_0^{(k)}}_2 + \norm{S_k(\vX^{(k-1)})}_2 \norm{\theta_1^{(k)}}_2\right)}_{u_k} \normFn{\vX^{(k-1)}} + \underbrace{\norm{\vb^{(k)}}}_{v_k} \\
    &= \sum_{\ell=1}^k \left[\prod_{p=\ell+1}^k u_p\right] v_\ell + \left[\prod_{\ell=1}^k u_p\right] \normFn{\vX^{(0)}}
    \end{align*}
\end{proof}

\subsection{Proof of Lemma \ref{statsconstantlemma}}\label{app:proofstatsconstant}

With the same reasoning as before, we have
$$
\std(C_i/C) \leq \max_i |C_i/C - 1|
$$
we are going to bound the right hand side.

The feature $C_i$ can be expressed as a sum over subsets of $r-1$ indices:
\begin{equation}\label{eq:Ci}
    C_i = \sum_{\mathcal{S} \subset ([1,n]\backslash \{i\})^{r-1}} 1_{G_{\{i\}\cup \mathcal{S}} \simeq H} 
\end{equation}
where $G_\mathcal{S}$ refers to the subgraph of $G$ corresponding to subset of nodes $\mathcal{S}$, and $1_{F \simeq H} = 1$ if $F$ and $H$ are isomorphic and $0$ otherwise.

By exchangeability, it is obvious that $C = \mathbb{E}(C_i)$ does not depend on the node $i$. For $\mathcal{S}$ of size $r$, since $G_\mathcal{S}$ is also an E-R graph of size $r$, we have $\mathbb{E}1_{G_{\{i\}\cup \mathcal{S}} \simeq H} = r! p^{m_H} (1-p)^{r(r-1)/2 - m_H}$ where $m_H$ is the number of edges in $H$, such that
\begin{equation}\label{eq:C}
    C = \binom{n-1}{r-1} r! p^{m_H} (1-p)^{r(r-1)/2 - m_H} = O(n^{r-1})
\end{equation}

We are going to bound the deviation of $C_i/C$ from its expectation of $1$ using McDiarmid's concentration inequality \citep{mcdiarmid}, which requires to bound the variation of $C_i/C$ when one edge $e_{k\ell}$ flips, and sum the square of these variations over all indices $k,\ell$.

If $e_{k \ell}$ flips and $k,\ell$ are different from $i$, then the terms affected are those where $\mathcal{S}$ contains both $k,\ell$. There are $\binom{n-3}{r-3} = O(n^{r-3})$ such terms, such that the deviation of $C_i/C$ is of order $1/n^2$. There are $O(n^2)$ such indices $k,\ell$, such that the sum of the square of the deviation has a term $O(n^2/n^4) = O(1/n^2)$.

If $e_{k \ell}$ flips and $k=i$, then the terms affected are those where $\mathcal{S}$ contains $\ell$. There are $\binom{n-2}{r-2} = O(n^{r-2})$ such terms, such that the deviation of $C_i/C$ is of order $1/n$. There are $O(n)$ such indices $k,\ell$, such that the sum of the square of the deviation has a term $O(n/n^2) = O(1/n)$. Similarly when $\ell=i$.

Hence, by applying McDiarmid's inequality, with probability $1-\delta$, $|C_i/C - 1|$ is bounded by $\sqrt{\log(1/\delta)/n}$. Using a union bound over $i$, we obtain the result.

\subsection{Proof of Lemma \ref{lem:u}}

A simple approximation of $1-e^{-x}$ yields
$$
p_{ij}  = e^{u_i + u_j} + O((e^{u_i + u_j})^2) = e^{u_i + u_j} + O(e^{-4c_n})
$$
Then,
\begin{align*}
    \mathbb{E}(d_i \mid \mathbf{u}) &= \sum_j p_{ij} = \sum_j e^{u_i + u_j} + O(ne^{-4c_n}) \\
    &\underbrace{\left(\sum_j e^{u_j}\right)}_C e^{u_i} + O(ne^{-4c_n})
\end{align*}

\section{Variational motivation for the training objective}
\label{app:elbo}

The following derivation provides variational motivation for the joint training objective in Equation
\ref{eq:joint_loss}. We follow the DiGress framework
\citep{vignac2023digressdiscretedenoisingdiffusion} and highlight the additional terms arising
from the continuous latent component $\mathbf{u}$.

A key structural feature of our forward process is that the discrete transition matrices
$\bar{\mathbf{Q}}_E^{t,ij}$ are parameterized by the \emph{clean} latent vector $\mathbf{u}^0$
rather than the noisy $\mathbf{u}^t$ this is what gives the forward process its closed-form
factorization (Section~\ref{jointdiffusionwithlsp}). As a consequence, the terminal distribution
of the forward process conditioned on a training graph $G^0$ has a discrete component that
depends on $\mathbf{u}^0$: edges converge to $\mathrm{Bernoulli}(p_{ij}^0)$ where
$p_{ij}^0 = 1 - e^{-e^{u_i^0 + u_j^0}}$. The sampling process (Algorithm~\ref{alg:sample_jointlsp}),
on the other hand, draws $\mathbf{u}^T \sim \mathcal{N}(\mu_{\mathrm{LSP}}, \sigma_{\mathrm{LSP}}^2 I)$
and then samples edges from $\mathrm{Bernoulli}(p_{ij}^T)$ where $p_{ij}^T$ depends on
$\mathbf{u}^T$. These two terminals coincide only in expectation over $\mathbf{u}^0$. We
therefore do not claim that \ref{eq:joint_loss} is an exact ELBO for the generative model
in Algorithm~\ref{alg:sample_jointlsp}; rather, following the pragmatic stance of DiGress
\citep{vignac2023digressdiscretedenoisingdiffusion}, we treat the objective as a well-motivated
denoising loss whose variational structure we describe below.

Let the joint state at time $t$ be $z^t = (\mathbf{u}^t, \mathbf{E}^t)$, with clean graph
$z^0 = (\mathbf{u}^0, \mathbf{E}^0)$ and terminal state $z^T = (\mathbf{u}^T, \mathbf{E}^T)$.

\paragraph{Joint forward factorization.}
Because the continuous and discrete components are corrupted independently, the forward process
factorizes as
\begin{equation}
q(z^t \mid z^0) = q(\mathbf{u}^t \mid \mathbf{u}^0)\, q(\mathbf{E}^t \mid \mathbf{E}^0, \mathbf{u}^0)
\label{eq:joint_forward_factorization}
\end{equation}
The continuous component follows the shifted-scaled Gaussian diffusion of
Appendix~\ref{appendix:ddpm_shifted}:
\begin{equation}
q(\mathbf{u}^t \mid \mathbf{u}^0)
=
\mathcal{N}\!\left(
\mathbf{u}^t \;\middle|\;
\mu_{\mathrm{LSP}} + \sqrt{\bar{\alpha}_t}(\mathbf{u}^0 - \mu_{\mathrm{LSP}}),\;
\sigma_{\mathrm{LSP}}^2(1-\bar{\alpha}_t)I
\right)
\end{equation}
The discrete component follows from the pair-specific forward marginals of
Equation \ref{eq:q_Et_joint}:
\begin{equation}
q(\mathbf{E}^t \mid \mathbf{E}^0, \mathbf{u}^0)
=
\prod_{1 \leq i,j \leq n}
q(e_{ij}^t \mid e_{ij}^0,\, \mathbf{u}^0)
\end{equation}

\paragraph{Evidence lower bound.}
Following \citep{vignac2023digressdiscretedenoisingdiffusion}, applying Jensen's inequality with
$q(z^{1:T} \mid z^0)$ as the variational distribution gives the ELBO:
\begin{equation}
\log p_\theta(G)
\geq
\log p(n_G)
-
D_{\mathrm{KL}}\!\left[q(z^T \mid G)\,\|\,p(z^T \mid n_G)\right]
-
\sum_{t=2}^{T} L_t(G)
+
\mathbb{E}_{q(z^1 \mid G)}\!\left[\log p_\theta(G \mid z^1)\right]
\label{eq:joint_elbo}
\end{equation}
where
\begin{equation}
L_t(G)
=
\mathbb{E}_{q(z^t \mid G)}
\left[
D_{\mathrm{KL}}\!\left(
q(z^{t-1} \mid z^t, z^0)
\,\|\,
p_\theta(z^{t-1} \mid z^t)
\right)
\right]
\label{eq:joint_lt}
\end{equation}

\paragraph{Terminal prior KL.}
With the $\mathbf{u}^0$-keyed forward process, the training terminal factors as
\begin{equation}
p(z^T \mid n_G)
=
\mathcal{N}(\mu_{\mathrm{LSP}},\,\sigma_{\mathrm{LSP}}^2 I)
\times
\prod_{1 \leq i,j \leq n} \mathrm{Bernoulli}(p_{ij}^0)
\end{equation}
where $p_{ij}^0 = 1 - \exp(-e^{u_i^0 + u_j^0})$ depends on the clean graph. The prior KL
decomposes as
\begin{multline}
D_{\mathrm{KL}}\!\left[q(z^T \mid G)\,\|\,p(z^T \mid n_G)\right]
=
D_{\mathrm{KL}}\!\left[q(\mathbf{u}^T \mid \mathbf{u}^0)\,\|\,\mathcal{N}(\mu_{\mathrm{LSP}}, \sigma_{\mathrm{LSP}}^2 I)\right]
\\
+\,
D_{\mathrm{KL}}\!\left[q(\mathbf{E}^T \mid \mathbf{E}^0, \mathbf{u}^0)\,\|\,\prod_{1 \leq i,j \leq n}\mathrm{Bernoulli}(p_{ij}^0)\right]
\label{eq:joint_prior_kl}
\end{multline}
Both terms vanish as $\bar{\alpha}_T \to 0$ for the $\mathbf{u}^0$-keyed forward process: the
continuous component converges to $\mathcal{N}(\mu_{\mathrm{LSP}}, \sigma_{\mathrm{LSP}}^2 I)$
and each row of $\bar{\mathbf{Q}}_E^{T,ij}$ converges to $(\mathbf{p}_{ij}^0)^\top$, so
$q(e_{ij}^T \mid e_{ij}^0, \mathbf{u}^0)$ converges to $\mathrm{Bernoulli}(p_{ij}^0)$
independently of $e_{ij}^0$. Note that this vanishing holds for the training terminal
conditioned on $\mathbf{u}^0$, not for the unconditional sampling terminal.

\paragraph{Decomposition of the diffusion loss.}
Because the forward posterior factorizes as
\begin{equation}
q(z^{t-1} \mid z^t, z^0)
=
q(\mathbf{u}^{t-1} \mid \mathbf{u}^t, \mathbf{u}^0)\,
q(\mathbf{E}^{t-1} \mid \mathbf{E}^t, \mathbf{E}^0, \mathbf{u}^0)
\label{eq:joint_posterior_factorization}
\end{equation}
the per-step loss $L_t(G)$ splits into independent continuous and discrete terms:
\begin{equation}
L_t(G) = L_t^u(G) + L_t^{\mathbf{E}}(G)
\label{eq:joint_kl_split}
\end{equation}
where $L_t^{\mathbf{E}}(G)$ is identical in form to the DiGress per-step loss
\citep{vignac2023digressdiscretedenoisingdiffusion}, and $L_t^u(G)$ is the additional term:
\begin{equation}
L_t^u(G)
=
\mathbb{E}_{q(\mathbf{u}^t \mid \mathbf{u}^0)}
\left[
D_{\mathrm{KL}}\!\left(
q(\mathbf{u}^{t-1} \mid \mathbf{u}^t, \mathbf{u}^0)
\,\|\,
p_\theta(\mathbf{u}^{t-1} \mid z^t)
\right)
\right]
\end{equation}

\paragraph{Practical training objective.}
As in DDPM \citep{ddpm} and DiGress \citep{vignac2023digressdiscretedenoisingdiffusion}, we do not
optimize the ELBO term-by-term. Instead, we optimize the simplified joint objective of
Equation \ref{eq:joint_loss}, replacing the per-step KL terms with cross-entropy and MSE losses
respectively. The prior KL terms are dropped since they vanish for the $\mathbf{u}^0$-keyed
forward process by construction.

\section{Additional experiments}
\label{additionalresults}
\subsection{DiGress on Tree dataset} In Table \ref{tab:treedataset} we also present results on the Tree dataset \citep{graphrnn} comparing DiGress with no aux features and our JointLSP with DiGress. We can observe that our proposed method outperforms DiGress on 3 out of 4 metrics.
\begin{table}[h]
\centering
\caption{Results on Tree dataset}
\label{tab:treedataset}
\resizebox{\linewidth}{!}{%
\begin{tabular}{@{}ccccc@{}}
\toprule
Method           & Degree                  & Assortativity & PLE           & CPL           \\ \midrule
DiGress          & 1073.41 ± 38.35         & 33.89 ± 0.81  & 242.14 ± 0.63 & \textbf{222.38 ± 6.51} \\
JointLSP+DiGress & \textbf{204.17 ± 26.83} & \textbf{15.45 ± 1.58}  & \textbf{134.26 ± 2.03} & 231.38 ± 1.19 \\ \bottomrule
\end{tabular}%
}
\end{table}

\subsection{Graph generation using JointLSP+SparseDiff}\label{sparsediffresults}
We augment SparseDiff \citep{qin2025sparsediff} with our prior and joint diffusion process,
referring to the result as JointLSP+SparseDiff. SparseDiff is particularly efficient at larger
graph sizes, processing only randomly selected node pairs at each step, which is why we use it as
the baseline for $n=500$ and the Ego and Community datasets.

The results on CoraChameleon across $n\in\{10,20,50,100,500\}$ are reported in Table
\ref{tab:sparsediff}. JointLSP+SparseDiff consistently improves over SparseDiff on most
metrics, with gains most pronounced on triangle count and assortativity. Notably, the improvement
holds at $n=500$, where the fixed-point pathology is most severe and the edge-independent prior
has the greatest influence on generated statistics.

The results on Ego and Community are reported in Table \ref{tab:egocommtable}. These datasets remain
challenging for both methods, as indicated by the uniformly large MMD ratios. This is likely due
in part to the presence of self-loops in Ego and the lack of auxiliary spectral features needed to
capture community structure. Nevertheless, JointLSP+SparseDiff improves over SparseDiff on all
metrics except clustering on Ego and assortativity on Community. Taken together, these results
support our main claim: the benefit of the proposed method is not that it universally optimizes
generation quality, but that it addresses a fundamental limitation of edge-independent diffusion
in the no-node feature setting.

\begin{table}[t]
\centering
\caption{MMD ratios on CoraChameleon datasets for SparseDiff and JointLSP+SparseDiff under
no-node features. Results are mean$\pm$std over 3 random seeds.
Bold: better of SparseDiff vs.\ JointLSP+SparseDiff.}
\label{tab:sparsediff}
\resizebox{\linewidth}{!}{%
\begin{tabular}{@{}lcccccc@{}}
\toprule
 & Degree $\downarrow$ & Clustering $\downarrow$ & Triangles $\downarrow$
 & Assortativity $\downarrow$ & PLE $\downarrow$ & CPL $\downarrow$ \\
\midrule
\multicolumn{7}{l}{\textit{SparseDiff (no features)}} \\
\quad $n=10$  & $17.23_{\pm 0.87}$ & $\mathbf{1.99}_{\pm 0.02}$ & $6.68_{\pm 0.07}$  & $69.77_{\pm 18.17}$ & $431.13_{\pm 133.46}$ & $114.33_{\pm 21.79}$ \\
\quad $n=20$  & $49.10_{\pm 2.08}$ & $1.41_{\pm 0.00}$          & $3.67_{\pm 0.43}$  & $30.17_{\pm 4.03}$  & $27.25_{\pm 2.67}$    & $0.92_{\pm 0.12}$  \\
\quad $n=50$  & $17.51_{\pm 0.98}$ & $1.55_{\pm 0.04}$          & $25.87_{\pm 0.77}$ & $60.65_{\pm 4.84}$  & $\mathbf{4.98}_{\pm 0.59}$    & $\mathbf{3.69}_{\pm 0.39}$  \\
\quad $n=100$ & $30.43_{\pm{0.15}}$           &   $3.43_{\pm{0.10}}$                        & $66.65_{\pm {0.76}}$                 & $50.44_{\pm{2.34}}$                  & $15.33_{\pm{0.10}}$                    & $18.92_{\pm{0.10}}$                   \\
\quad $n=500$ &  $3.93_{\pm0.01}$                 & $3.23_{\pm0.02}$                        & $113.81_{\pm7.41}$                 & $39.04_{\pm0.14}$                 &       $\textbf{2.13}_{\pm0.00}$           & $\textbf{7.93}_{\pm0.00}$                   \\
\midrule
\multicolumn{7}{l}{\textit{JointLSP+SparseDiff (no features)}} \\
\quad $n=10$  & $\mathbf{14.18}_{\pm 0.64}$ & $2.00_{\pm 0.01}$          & $\mathbf{6.43}_{\pm 0.14}$  & $\mathbf{11.38}_{\pm 7.16}$ & $\mathbf{37.70}_{\pm 14.13}$  & $\mathbf{37.19}_{\pm 3.19}$  \\
\quad $n=20$  & $\mathbf{29.08}_{\pm 1.11}$ & $\mathbf{1.40}_{\pm 0.00}$ & $\mathbf{2.19}_{\pm 0.23}$  & $\mathbf{15.31}_{\pm 0.95}$ & $\mathbf{15.75}_{\pm 1.29}$   & $\mathbf{0.54}_{\pm 0.12}$            \\
\quad $n=50$  & $\mathbf{13.29}_{\pm 0.32}$ & $\mathbf{1.49}_{\pm 0.01}$ & $\mathbf{15.11}_{\pm 0.34}$ & $\mathbf{36.18}_{\pm 1.16}$ & $6.29_{\pm 0.19}$             & $5.76_{\pm 0.16}$            \\
\quad $n=100$ & $\mathbf{19.07}_{\pm 0.49}$          & $\mathbf{3.07}_{\pm 0.07}$          & $\mathbf{36.00}_{\pm 0.84}$          & $\mathbf{24.88}_{\pm 1.18}$          & $\mathbf{11.14}_{\pm 0.23}$            & $\mathbf{16.24}_{\pm 0.21}$           \\
\quad $n=500$ & $\textbf{3.82}_{\pm0.01}$                          & $\textbf{2.99}_{\pm0.02}$                       & $\textbf{35.41}_{\pm 4.21}$                      & $\textbf{28.89}_{\pm0.27}$                          &      $2.20_{\pm0.00}$                     &    $8.37_{\pm0.00}$                         \\
\bottomrule
\end{tabular}%
}
\end{table}

\begin{table}[h]
\centering
\caption{Comparison of graph generation with JointLSP+SparseDiff and SparseDiff on Ego and
Community datasets. Results are mean$\pm$std over 3 random seeds.}
\label{tab:egocommtable}
\resizebox{\linewidth}{!}{%
\begin{tabular}{@{}lcccccc@{}}
\toprule
 & Degree $\downarrow$ & Clustering $\downarrow$ & Triangles $\downarrow$
 & Assortativity $\downarrow$ & PLE $\downarrow$ & CPL $\downarrow$ \\
\midrule
\multicolumn{7}{l}{\textit{Ego}} \\
\quad SparseDiff          & $768.06_{\pm 107.44}$ & $\mathbf{14.78}_{\pm 0.25}$ & $85.20_{\pm 12.21}$ & $1715.86_{\pm 53.10}$ & $1696.14_{\pm 183.27}$ & $183.07_{\pm 31.15}$ \\
\quad JointLSP+SparseDiff & $\mathbf{302.50}_{\pm 51.36}$ & $15.48_{\pm 0.57}$ & $\mathbf{12.41}_{\pm 1.25}$ & $\mathbf{785.72}_{\pm 141.52}$ & $\mathbf{750.94}_{\pm 145.86}$ & $\mathbf{28.95}_{\pm 6.31}$ \\
\midrule
\multicolumn{7}{l}{\textit{Community}} \\
\quad SparseDiff          & $91.57_{\pm 24.34}$ & $6.86_{\pm0.63}$  & $8.12_{\pm0.96}$  & $0.99_{\pm 1.27}$  & $226.18_{\pm49.36}$  & $21917.62{\pm157.43}$ \\
\quad JointLSP+SparseDiff & $\textbf{38.70}_{\pm{14.18}}$  & $\textbf{5.08}_{\pm{0.44}}$  & $\textbf{1.54}_{\pm{0.44}}$  & $\textbf{9.02}_{\pm{6.79}}$  & $\textbf{105.09}_{\pm{39.66}}$  & $\textbf{18251.36}_{\pm{550.46}}$ \\
\bottomrule
\end{tabular}%
}
\end{table}

\subsection{DiGress with degree features}
In the JointDiffusion+LSP proposal we derive the latent weights for each node by using 
the node degree as a proxy. This naturally leads to the question: is this similar to 
supplying node degree as auxiliary features to the denoiser, and whether 
DiGress~\citep{vignac2023digressdiscretedenoisingdiffusion} could also benefit from such 
augmentation? To answer this question, we conduct an experiment where we compute node 
degrees from the noisy graph at each denoising step and feed them as auxiliary features. 
More specifically, we compute the normalized degree $\tilde{d}_i = d_i / (n - 1)$ as a 
node-level feature, and the mean $\mu_d$ and standard deviation $\sigma_d$ of 
$\tilde{d}_i$ across nodes as graph-level features, recomputed from the noisy adjacency 
$A_t$ at every reverse step. The results are shown in Table~\ref{tab:digress_degree}. We can observe that JointLSP+DiGress outperforms DiGress+Degree features on metrics like degree, triangles, PLE and CPL. Thus, DiGress augmented with degree 
features does not recover the performance of JointLSP. This is consistent with 
Lemma~\ref{statsconstantlemma}: at high noise levels $A_t$ is approximately E-R, so all 
$\tilde{d}_i \approx \bar{p}$ and the features collapse to a near-constant, carrying no 
structural signal. The key distinction is that the LSP does not compute degree on the 
noisy graph, it supplies this information inherently by construction. The degree-informed heterogeneity information is baked into the terminal prior itself via 
$u^0$, making node-level variation available throughout the entire reverse trajectory 
rather than only when the graph is sufficiently clean.

\begin{table}[h]
\centering
\caption{DiGress + degree features on CoraChameleon-Node100. JointLSP+DiGress ($k=0$, no features) is included for direct comparison. Bold: best result excluding the $k=17$ ceiling.}
\label{tab:digress_degree}
\resizebox{\linewidth}{!}{%
\begin{tabular}{@{}lcccccc@{}}
\toprule
 & Degree $\downarrow$ & Clustering $\downarrow$ & Triangles $\downarrow$
 & Assortativity $\downarrow$ & PLE $\downarrow$ & CPL $\downarrow$ \\
\midrule
\multicolumn{7}{l}{\textit{DiGress+Degree}} \\
\quad $k=1$  & $10.67_{\pm 0.14}$   & $\textbf{2.04}_{\pm 0.02}$  & $2.41_{\pm 0.14}$ & $\textbf{57.66}_{\pm 0.76}$  & $10.81_{\pm 0.17}$  & $45.54_{\pm 0.56}$  \\
\midrule
\multicolumn{7}{l}{\textit{JointLSP+DiGress}} \\
\quad $k=0$ & $\textbf{10.50}_{\pm 0.16}$ & ${2.07}_{\pm 0.10}$ & $\mathbf{2.33}_{\pm 0.15}$ & ${59.72}_{\pm 0.75}$ & $\textbf{10.79}_{\pm 0.19}$ & $\textbf{44.99}_{\pm 0.29}$ \\
\midrule
\multicolumn{7}{l}{\textit{DiGress + 17 features}} \\
\quad $k=17$ & $0.93_{\pm 0.00}$ & $1.95_{\pm 0.00}$ & $1.68_{\pm 0.00}$ & $3.55_{\pm 0.00}$ & $1.53_{\pm 0.00}$ & $1.92_{\pm 0.00}$ \\
\bottomrule
\end{tabular}%
}
\end{table}
\paragraph{DiGress+degree does not preserve node heterogeneity.} We compute the variance in node propensity scores as defined in Section~\ref{nodepropsection}, comparing
JointLSP+DiGress, DiGress, and DiGress+Degree — where node degree is supplied as both a node-level
and graph-level auxiliary feature — in Figure~\ref{fig:propexptsdegree}. JointLSP+DiGress
maintains consistently high node heterogeneity throughout the reverse trajectory, starting from
$t=T$ where the LSP prior already induces diverse connection propensities. DiGress+Degree, by
contrast, starts near zero at high noise levels and builds heterogeneity only as denoising
progresses. This confirms that degree features are reactive: they can only provide structural
signal once the graph is already partially denoised, whereas LSP provides it from initialization.
The two mechanisms are therefore fundamentally different, even though both draw on degree
information.

\begin{figure}[h]
    \centering
    \begin{subfigure}{0.47\linewidth}
        \centering
        \includegraphics[width=\linewidth]{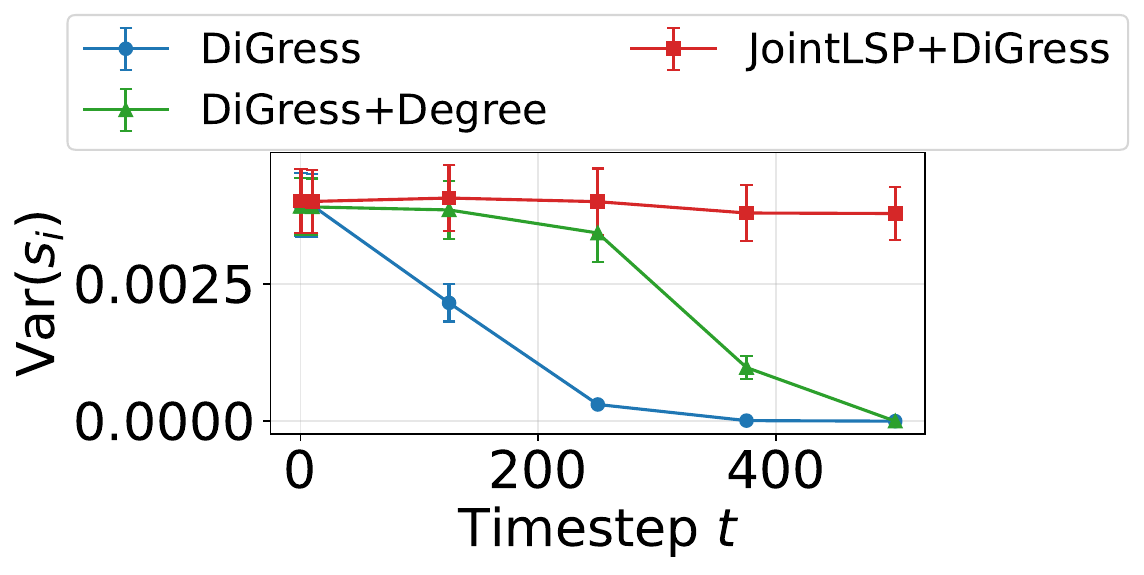}
        \caption{Forward Noising Process.}
          \label{fig:vfwdnormplot}
    \end{subfigure}
    \begin{subfigure}{0.47\linewidth}
        \centering
        \includegraphics[width=\linewidth]{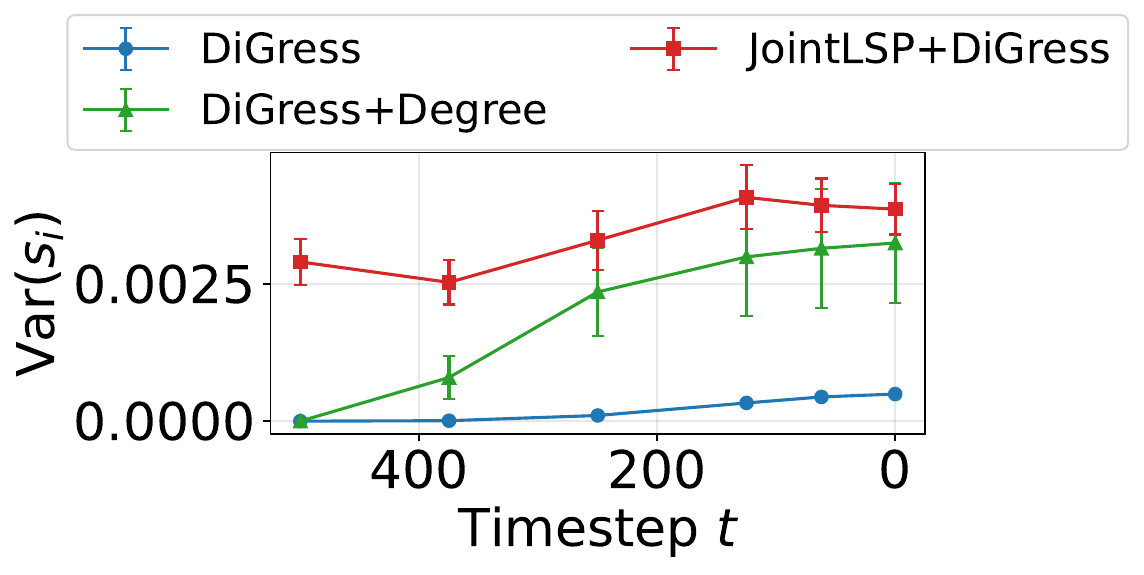}
        \caption{Reverse Denoising Process.}
        \label{fig:vreversedegree}
    \end{subfigure}
    \caption{We compare the variance in node propensity scores for JointLSP+DiGress, DiGress+degree and DiGress. We can observe that our proposal yields node heterogeneity throughout the diffusion trajectory even at highly noisy states as opposed to DiGress with degree features which only improve the node heterogeneity when the graph is already denoised.}
    \label{fig:propexptsdegree}
\end{figure}

\paragraph{Random node features}
Many works such as \citep{abboud2021surprisingpowergraphneural,sato2021randomfeaturesstrengthengraph,gonon2026universalityapproximationratesgraph} suggest incorporating random node features improve GNN expressivity, we test if this also helps improve generation quality in graph diffusion models. We verify this empirically on the CoraChameleon-Node100 dataset and compare DiGress with no aux features, DiGress+RND which denotes the random node features and JointLSP+Digress. The setup is as follows: The diffused node state is unchanged ($X = \mathbf{1}$), and at each denoising step we concatenate random features $r \in \mathbb{R}^{n \times d}$, drawn i.i.d. from $\mathcal{N}(0, I)$ with dim=8 to the node input. The results are presented in Table \ref{tab:rnd}, adding random node features does not help with generation because firstly RND leaves the corruption process and the terminal distribution unchanged: Lemma \ref{lem:ER} still applies and the terminal distribution is still E-R, and sampling still begins from a graph with no node-level structure. RND alters only the denoiser input, not where the reverse trajectory starts. Second, the heterogeneity it supplies is uncalibrated, as the features are drawn independently of the data, whereas LSP is calibrated from the training data.

\begin{table}[h]
\centering
\caption{Results on CoraChameleon-Node100 dataset comparing DiGress with random node features}
\label{tab:rnd}
\resizebox{\textwidth}{!}{%
\begin{tabular}{@{}ccccccc@{}}
\toprule
Method  & Degree            & Clustering & Triangles  & Assortativity & PLE       & CPL        \\ \midrule
DiGress & \textbf{4.53±0.1} & 11.33±0.27 & 14.82±0.12 & 101.95±2.13   & 4.64±0.17 & 39.68±1.14 \\
DiGress+RND      & 5.48±0.13           & 12.50±1.18         & 10.13±0.04         & 127.06±6.20         & \textbf{4.14±0.41} & \textbf{24.01±1.57} \\
JointLSP+DiGress & \textbf{10.50±0.16} & \textbf{2.07±0.10} & \textbf{2.33±0.15} & \textbf{59.72±0.75} & 10.79±0.19         & 44.99±0.29          \\ \bottomrule
\end{tabular}%
}
\end{table}

\subsection{LSP improves generation quality throughout diffusion trajectory}
To further demonstrate the effectiveness of the proposed JointLSP framework, we track the MMD ratio ($r$) presented in Section \ref{experimentsgraphgeneration} at various timesteps $t \in \{500,50,0\}$ for both JointLSP+DiGress and DiGress on CoraChameleon-Node100 dataset. At each
timestep $t$, we use the network's predicted clean edge probabilities $\hat{\mathbf{p}}_{ij}^{E,t}$
to construct a graph and evaluate MMD against the test set. The results are shown in
Figure~\ref{fig:trajectory_mmd}. JointLSP+DiGress consistently achieves lower MMD ratios than DiGress at all timesteps across all three structural statistics except Degree, where we surmise the LSP over-corrects the degree when inducing node heterogeneity. At $t=500$, where the input is close to the terminal prior,
DiGress produces graphs with very high MMD as the ER prior provides no structural information making it difficult for the denoiser. JointLSP+DiGress already achieves substantially lower MMD
at this stage, confirming that the LSP along with the joint diffusion process provides a more informative starting point. As $t \to 0$ both models improve, but JointLSP+DiGress maintains a consistent advantage,
particularly on assortativity and triangle count. This supports our main claim that our proposal indeed improves the diffusion pipeline and does not result in the fixed-point pathology we have pointed out.

\begin{figure}[h]
    \centering
    \includegraphics[width=0.9\textwidth]{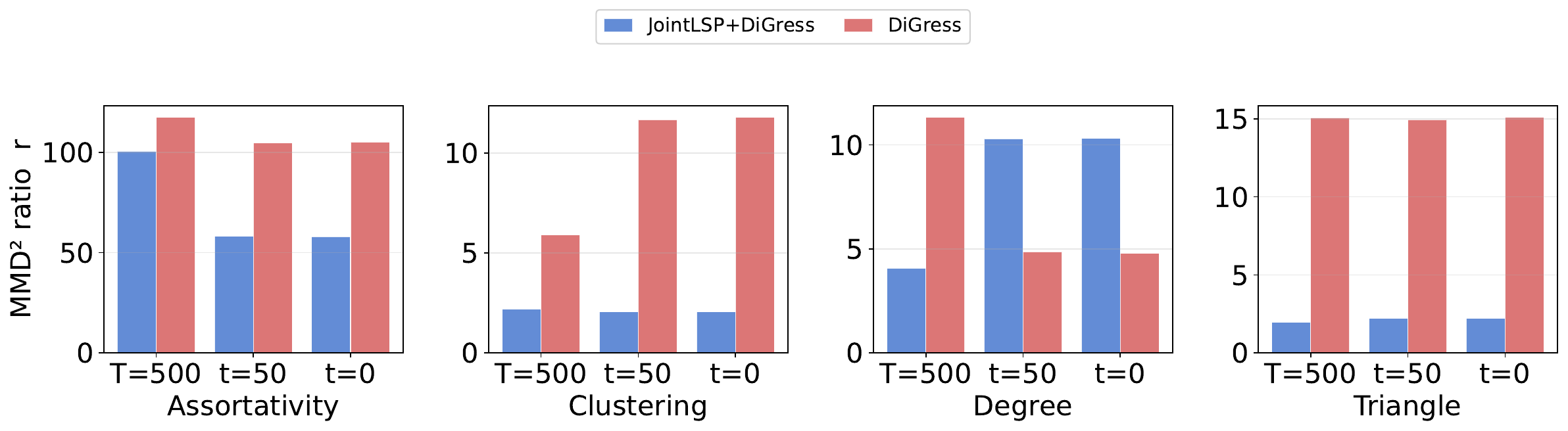}
    \caption{We track the MMD ratio across various diffusion timesteps for JointLSP+DiGress and DiGress. Our proposed framework consistently generates good quality graphs even at intermediate steps.}
    \label{fig:trajectory_mmd}
\end{figure}

\subsection{LSP induces node heterogeneity}
To verify that the proposed prior preserves node distinguishability, we simulate graphs (\(m=100\) graphs) with increasing node sizes \(n=\{10,20,50,100,500,1000\}\) under three random graph models:

\begin{enumerate}
    \item Dense Bernoulli baseline: \(A_{ij} \sim \texttt{Bern}(\rho_{\text{dense}})\), where \(\rho_{\text{dense}} = 0.1\).
    \item Sparse mean-field Bernoulli: \(A_{ij} \sim \texttt{Bern}(\rho_{\text{sparse}})\), where \(\rho_{\text{sparse}} = \frac{\bar{k}}{n-1}\) and \(\bar{k}=5\), and
    \item Latent Sociability Prior: \(A_{ij} \sim \texttt{Bern}(P_{ij})\), where $P_{ij} = 1-e^{(-ce^{u_i+u_j})}$. 
    \end{enumerate}
    
    To quantify node-level heterogeneity, we measure the within-graph degree variance \(\mathrm{Var}_i(d_i^{(m)})\) and within-graph degree mean \(\mathbb{E}_i[d_i^{(m)}]\), and define the dispersion index
\begin{equation}
    \varphi = \frac{\mathrm{Var}_i[d_i^{(m)}]}{\mathbb{E}_i[d_i^{(m)}]},
\end{equation}
which we then average across all sampled graphs. This statistic captures how much node degrees vary within a graph. The results are shown in Figure \ref{fig:exp2}. As \(n \to 1000\), both the dense and sparse Bernoulli baselines yield \(\varphi \approx 1\), indicating that edge-independent models remain tightly constrained by their mean degree and therefore, by design, cannot induce meaningful node heterogeneity. In contrast, LSP consistently yields \(\varphi > 1\), showing that it preserves non-trivial degree variation as graph size increases. This confirms that, unlike ER-type models, LSP does not collapse to a homogeneous fixed-point.
\begin{figure}[h]
    \centering
    \includegraphics[width=0.65\textwidth]{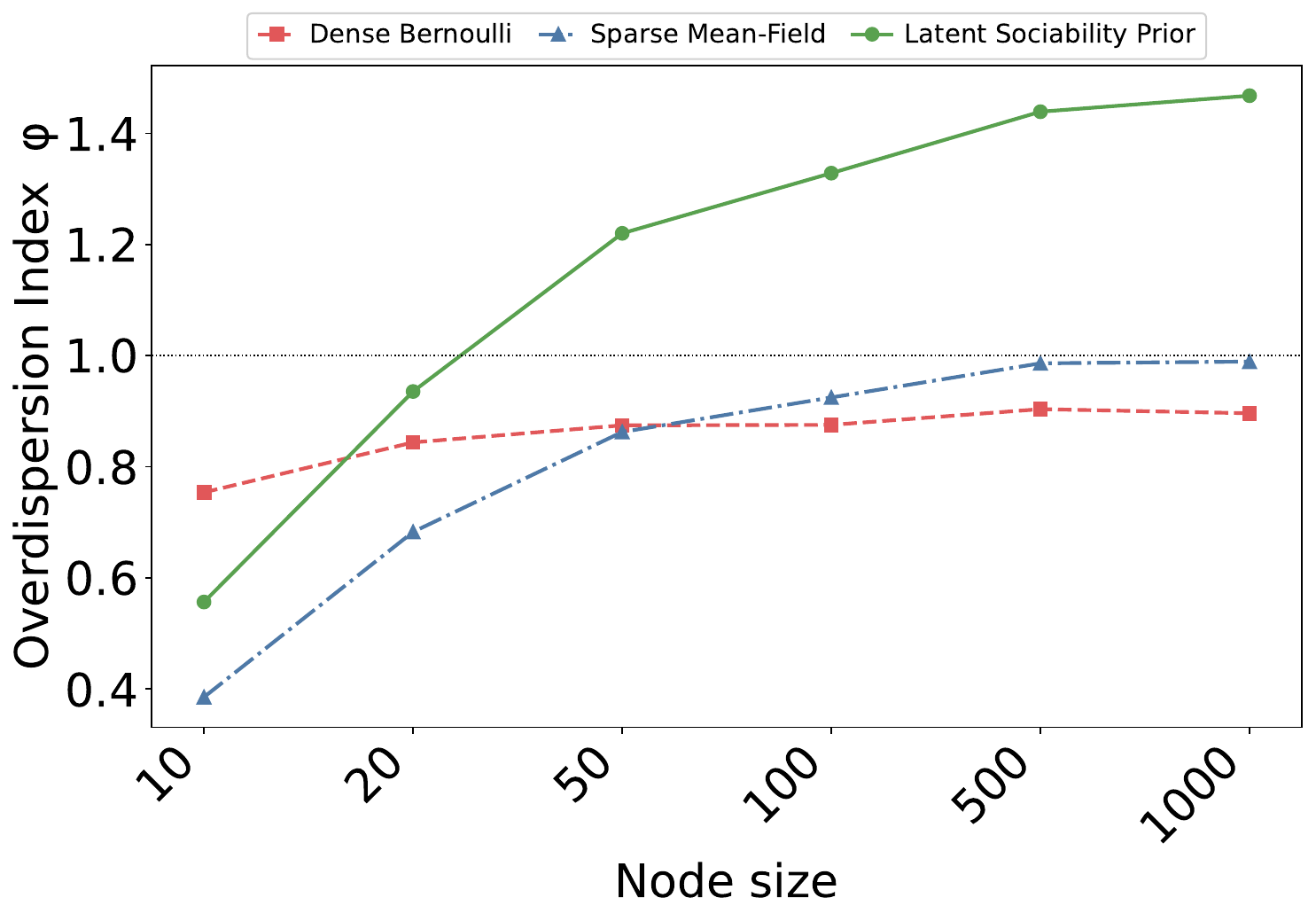}
    \caption{Dispersion index $\varphi$ as a function of node size under three prior models. Unlike edge-independent Bernoulli baselines, LSP preserves non-trivial node heterogeneity as graph size increases.}
    \label{fig:exp2}
\end{figure}
\subsection{LSP can generate higher triangle and cycle counts} \label{app:lemma43sim}
We conduct a second experiment to study how higher-order graph statistics scale with increasing graph size under a fixed mean-degree constraint. This distinction is important, since large triangle or cycle counts can be obtained trivially by generating dense graphs. Specifically, for node sizes \(n=\{10,20,50,100,500,1000\}\), we simulate \(100\) graphs from each random graph model while fixing the mean degree to the same target value. For each graph, we measure the number of triangles, 4-cycles, and 4-cliques, and study how these quantities scale as 
$n$ increases. The results, shown in Figure \ref{fig:higherorderstatscounts}, indicate that LSP consistently produces substantially larger higher-order structure counts than a sparse mean-field Bernoulli prior under the same mean-degree budget. This shows that the gain is not a consequence of simply generating denser graphs, but of inducing non-trivial dependencies that support richer local structure.

\begin{figure}[h]
    \centering
    \begin{subfigure}{0.28\linewidth}
        \centering
        \includegraphics[width=\linewidth]{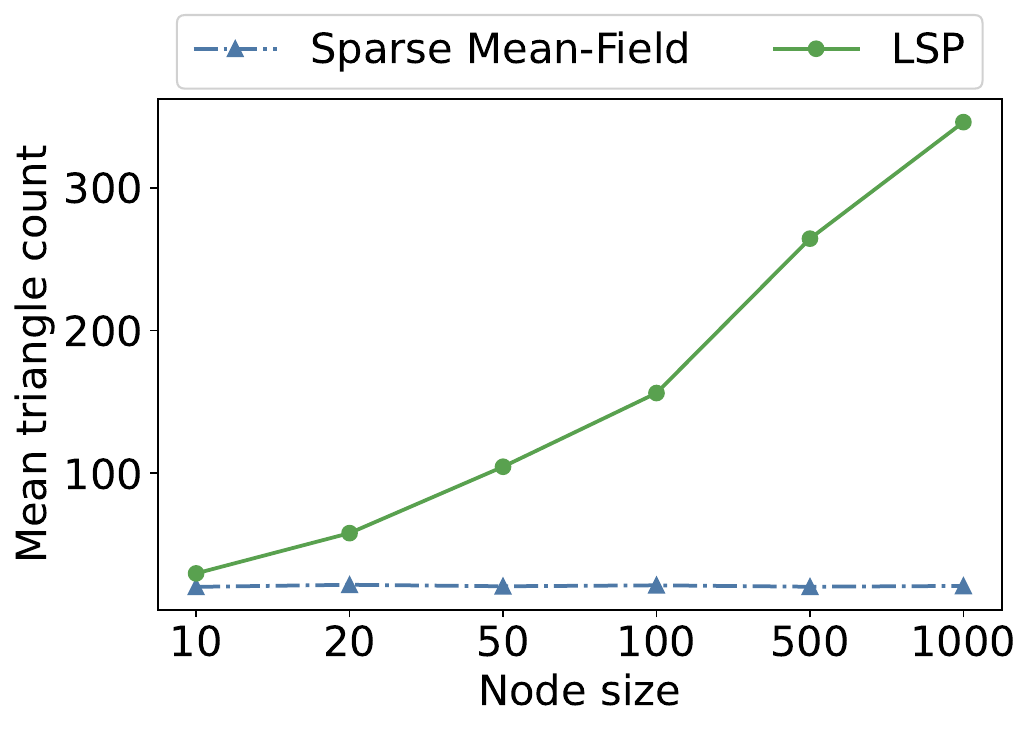}
        \caption{Triangle Count.}
          \label{fig:tricount}
    \end{subfigure}
    \begin{subfigure}{0.28\linewidth}
        \centering
        \includegraphics[width=\linewidth]{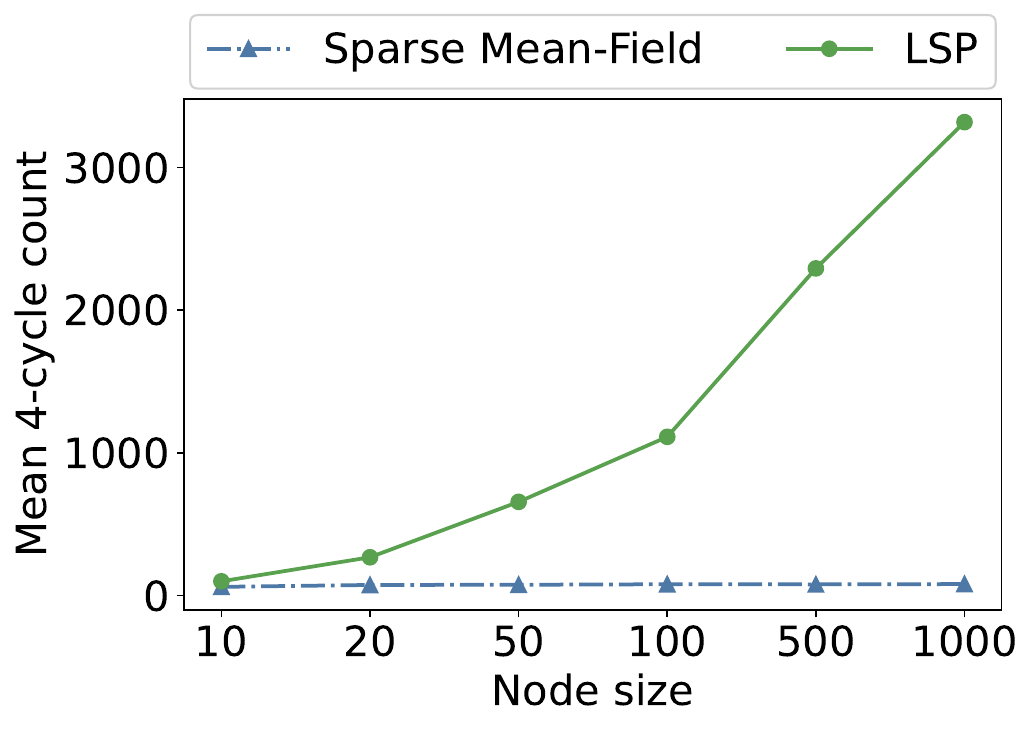}
        \caption{4-Cycle Count.}
        \label{fig:4cyclecount}
    \end{subfigure}
        \begin{subfigure}{0.28\linewidth}
        \centering
        \includegraphics[width=\linewidth]{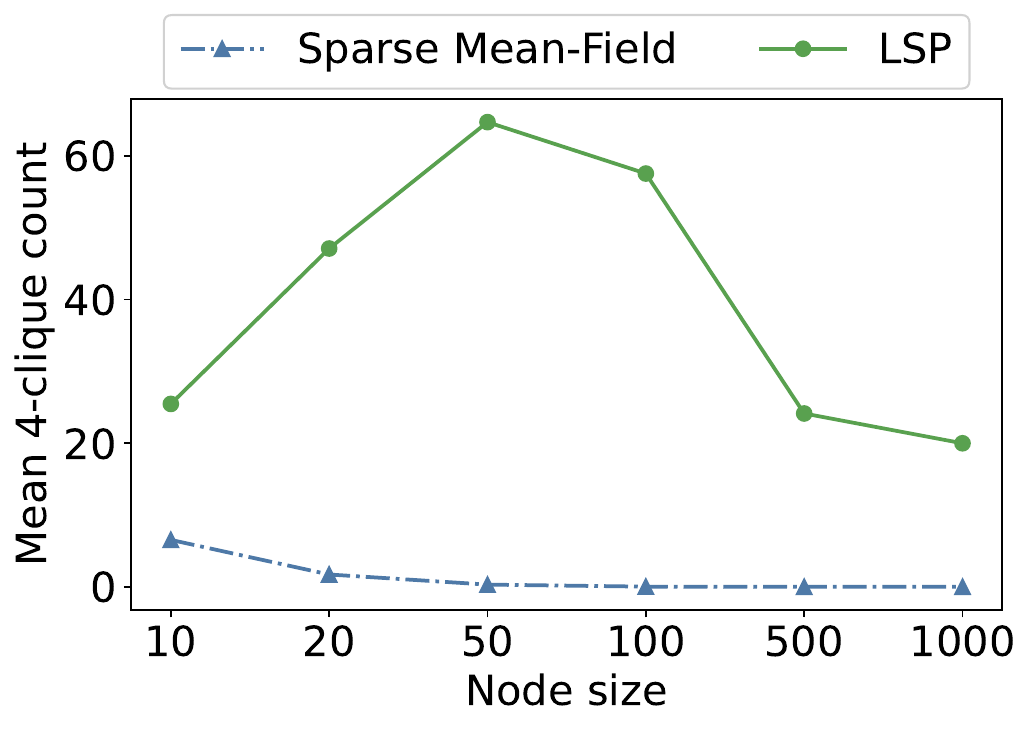}
        \caption{4-Clique Count.}
        \label{fig:4cliquecount}
    \end{subfigure}
    \caption{We simulate graphs using the sparse Bernoulli and our proposed LSP and compute how many triangles, 4-cycles and 4-cliques can the prior generate at a fixed mean-degree resolution while the number of nodes is increased.}
    \label{fig:higherorderstatscounts}
\end{figure}

\subsection{Structural features are uninformative for E-R graphs} \label{{app:lemma43sim}}
In Figure~\ref{fig:lemma4.3} we empirically validate Lemma~\ref{statsconstantlemma}. We simulate 100 graphs of node sizes $n \in \{100, 500, 1000, 2000, 5000\}$ under two 
priors at matched edge density $\rho = 0.1$. For E-R, each edge is drawn independently 
with probability $p = 0.1$. For the LSP, we fix $\sigma_\mathrm{LSP} = 1.0$ and solve for $\mu_\mathrm{LSP}$ such that $\mathbb{E}[p_{ij}] = 0.1$ using the calibration procedure of Algorithm~\ref{alg:calibrate_lsp}. Graphs are then sampled directly from the prior by drawing $u_i \sim \mathcal{N}(\mu_\mathrm{LSP}, \sigma^2_\mathrm{LSP})$ and connecting each pair $(i,j)$ with probability $p_{ij} = 1 - e^{-e^{u_i + u_j}}$. Note 
that no diffusion process is involved; we sample from the terminal prior directly to isolate the effect of the prior on structural feature statistics. For each graph, we compute the per-node 
triangle count $C_i$ and measure $\mathrm{std}(C_i / \bar{C})$ where 
$\bar{C} = \mathrm{mean}_i(C_i)$. As $n$ grows, this quantity vanishes for E-R graphs, 
confirming that structural features become increasingly uniform across nodes and therefore 
uninformative to a denoiser. The LSP maintains a stable, non-trivial spread across all 
graph sizes.

\begin{figure}[h]
    \centering
    \includegraphics[width=0.55\textwidth]{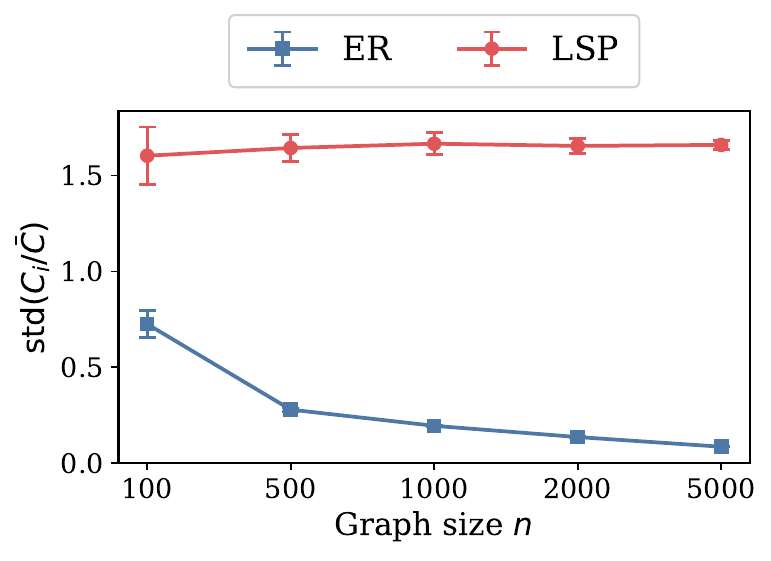}
    \caption{We empirically validate the hypothesis that in the large graph limit the structural features derived from an E-R are uninformative for the denoiser.}
    \label{fig:lemma4.3}
\end{figure}

\section{Dataset generation} \label{datasetgen}
\paragraph{CoraChameleon datasets.} Existing datasets like Community and Planar graphs which are usually used in evaluation of the graph diffusion models do not accurately reflect the characteristics of real-world graphs. Hence, we generate 5 new datasets by sampling from two real-world datasets Cora \citep{Cora} and Chameleon \citep{wikipediadatasets}. 

We apply 3 sampling strategies to the source graphs: random walk based where a biased random walk starts from a random node. The walk continues until $\sim1.5$x the target node size has been visited, then the BFS-largest connected component is taken. After half of the allowed attempts, the starting node is biased toward high-degree nodes to improve connectivity. Forest-fire \citep{leskovec2007graphevolutiondensificationshrinking} which starting from a seed node, each active node stochastically ignites a geometrically-distributed number of its un-burned neighbors with forward probability $p_\mathrm{forward} = 0.7$ and backward probability $p_{\mathrm{backward}} = 0.3$. And a dense-core sampling strategy that randomly picks top-$10\%$ of highest degree nodes in the source graph as candidate neighbours and greedily scores how many edges are already in the currently selected set and picks the highest-scoring candidate and adds it. A density check is done once target size nodes are collected, rejects the subgraph if it isn't connected or if its average degree is below minimum average degree (default 4.0). We use a ratio of 40\% Random walk, 40\% Forest-fire and 20\% Dense core to generate our synthetic graphs. Additionally 100 graphs are generated which are divided into 60/40 as test and validation sets respectively. The hyperparameters are tuned on the validation set.

\paragraph{Ego network} Similar to SparseDiff \citep{qin2025sparsediff} we also evaluate our models on the Ego dataset \citep{graphrnn} which is a 3-hop ego network extracted from the Citeseer dataset \citep{Citeseer}.

\paragraph{Community dataset} We also evaluate on the community dataset \citep{graphrnn} which consists of 2 communities generated by an E-R model with $n = |V|/2$ nodes and $p = 0.3$. The inter-community edges are added with probability $0.05|V|$ with uniform probability. In Table \ref{tab:datasetstats} we provide the data statistics of the graphs used for the generation task.

\begin{table}[h]
\centering
\caption{Statistics for the datasets used in the graph generation experiments}
\label{tab:datasetstats}
\resizebox{10cm}{!}{%
\begin{tabular}{@{}cccc@{}}
\toprule
Dataset               & \#Graphs & \#Nodes      & \#Edges         \\ \midrule
CoraChameleon-Node10  & 3000     & 10           & {[}14,45{]}     \\
CoraChameleon-Node20  & 3000     & 20           & {[}27,120{]}    \\
CoraChameleon-Node50  & 3000     & 50           & {[}71,300{]}    \\
CoraChameleon-Node100 & 3000     & 100          & {[}172,600{]}   \\
CoraChameleon-Node500 & 1400     & 500          & {[}1282,4630{]} \\
EGO                   & 757      & {[}50,399{]} & {[}57,1071{]}   \\
Community             & 510      & {[}60,160{]} & {[}231,1956{]}  \\ \bottomrule
\end{tabular}%
}
\end{table}

\begin{figure}[h]
    \centering
    \begin{subfigure}{0.49\linewidth}
        \centering
        \includegraphics[width=\linewidth]{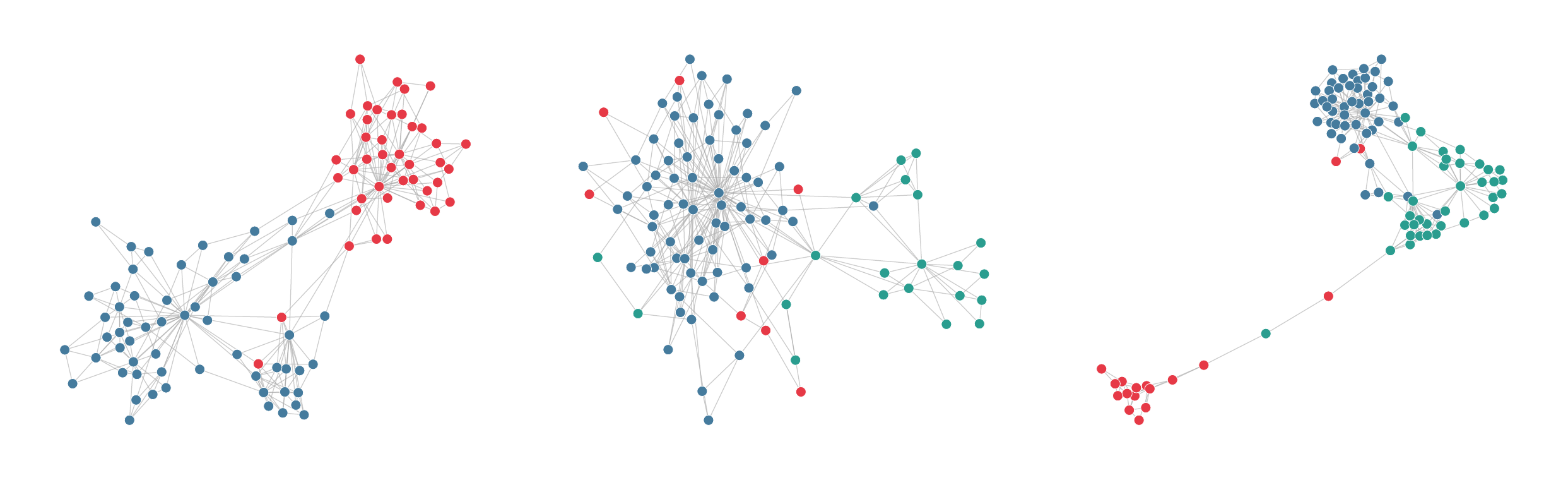}
        \caption{CoraChameleon-Node100}
          \label{fig:node50}
    \end{subfigure}
    \begin{subfigure}{0.49\linewidth}
        \centering
        \includegraphics[width=\linewidth]{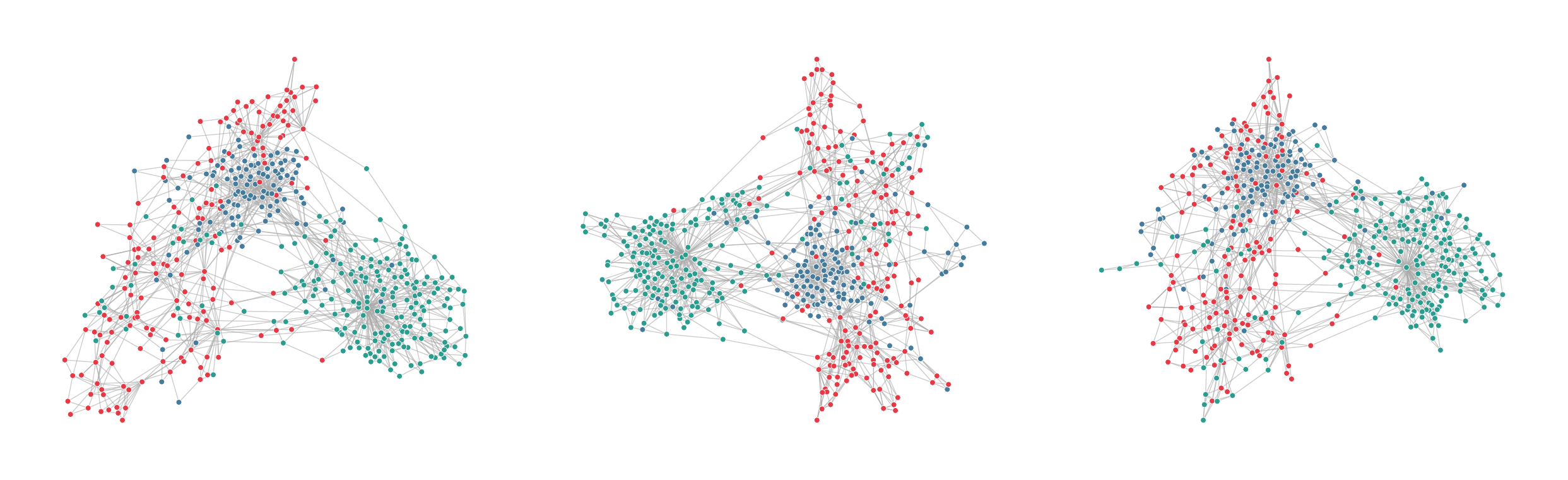}
        \caption{CoraChameleon-Node500}
        \label{fig:node500}
    \end{subfigure}
        \begin{subfigure}{0.49\linewidth}
        \centering
        \includegraphics[width=\linewidth]{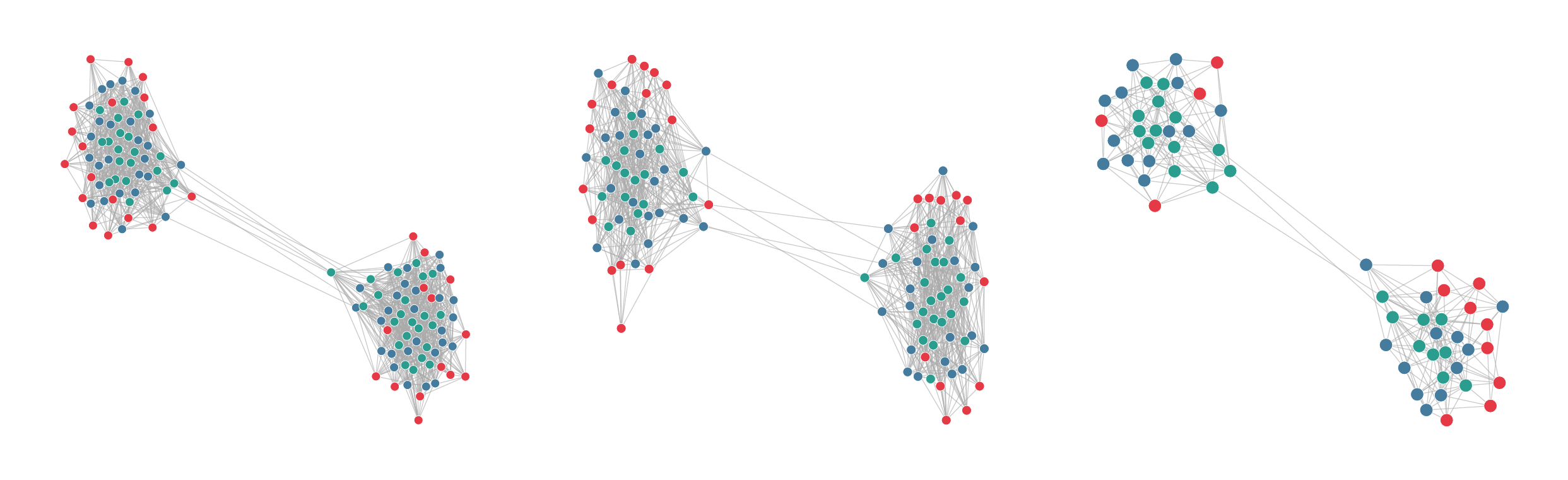}
        \caption{Community dataset}
        \label{fig:comm20}
    \end{subfigure}
    \caption{Sample graphs from the training set.}
    \label{fig:dataviz}
\end{figure}

\begin{figure}[h]
    \centering
    \begin{subfigure}{0.45\linewidth}
        \centering
        \includegraphics[width=\linewidth]{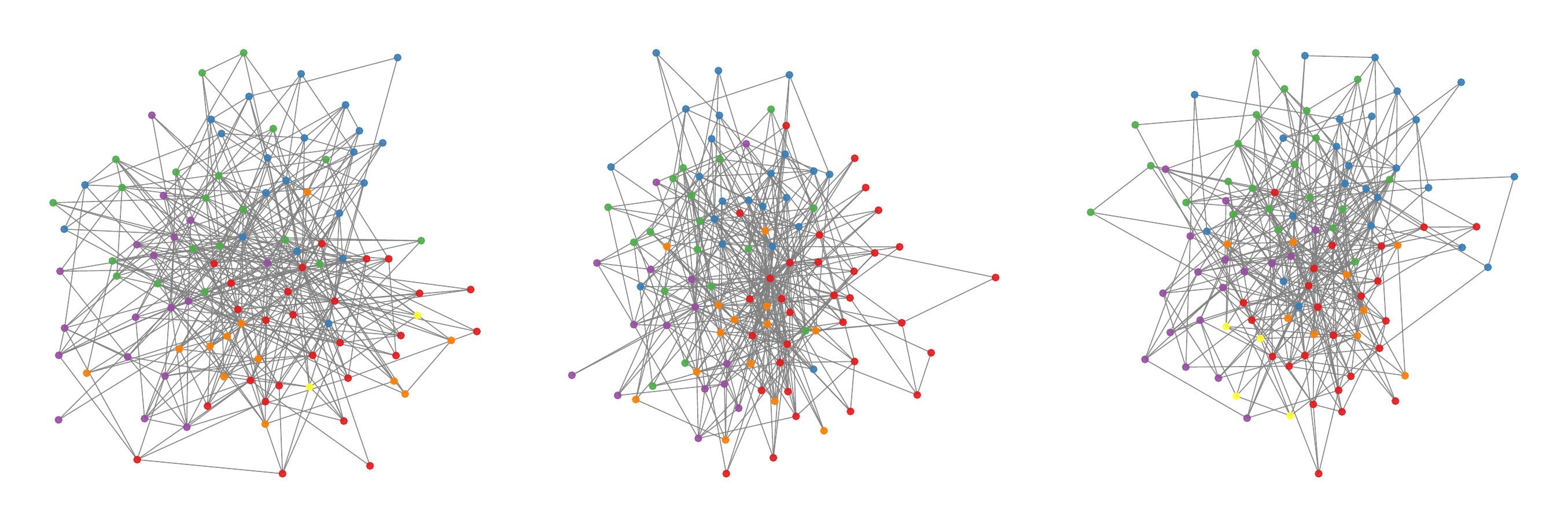}
        \caption{Sample CoraChameleon-Node100 graphs generated using JointLSP+DiGress.}
          \label{fig:lspnode100gen}
    \end{subfigure}
    \begin{subfigure}{0.45\linewidth}
        \centering
        \includegraphics[width=\linewidth]{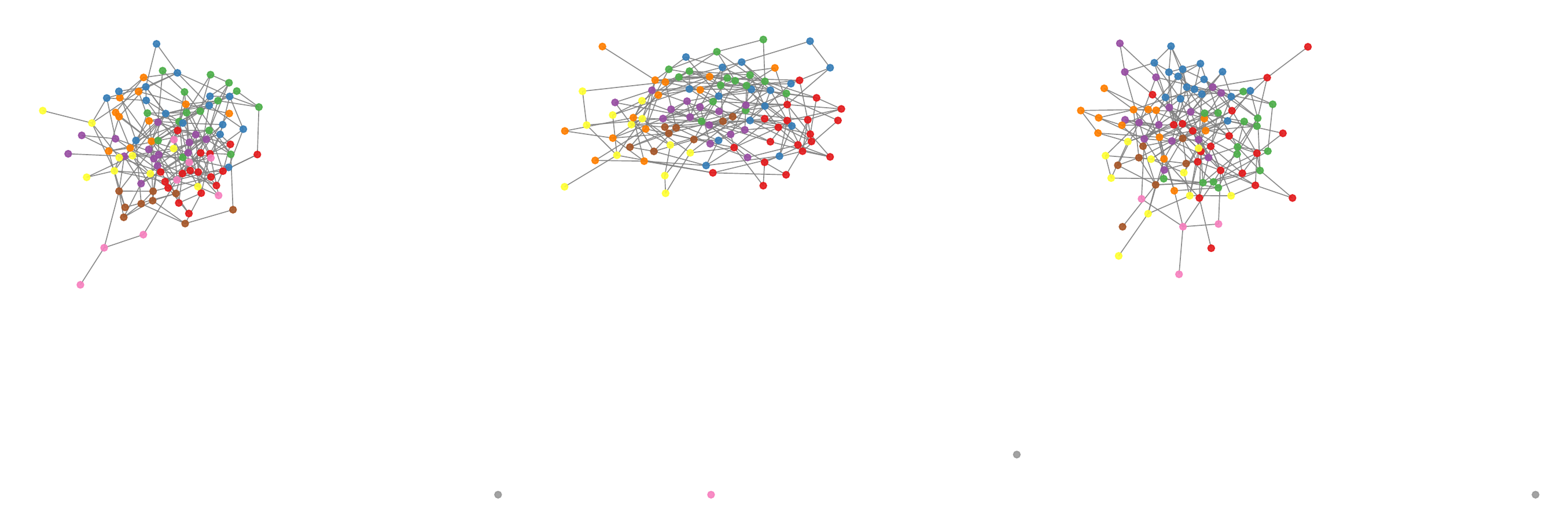}
        \caption{Sample CoraChameleon-Node100 graphs generated using DiGress.}
        \label{fig:dignode100}
    \end{subfigure}
    \caption{We visualize the generated graphs from both JointLSP+DiGress (a) and DiGress (b) on CoraChameleon-Node100 dataset. Clearly, we can observe that without node features the generation quality is impacted, but DiGress generates graphs which are dense and hairball-like even generating isolated nodes.}
    \label{fig:generatedgraphsN100}
\end{figure}

\subsection{Training details} \label{trainingdetails}
All experiments were implemented in PyTorch and PyTorch Geometric~\citep{Fey/Lenssen/2019}. 
Our model architecture follows DiGress and SparseDiff, using a Graph 
Transformer~\citep{dwivedi2021generalizationtransformernetworksgraphs} as the denoising 
network. All experiments were run on a single A40 GPU. Hyperparameters are provided in 
Table \ref{tab:hyperparams}. For each configuration, we train a single model and run 
inference three times with different random seeds; reported mean and standard deviation 
are computed over these three inference runs. This evaluation setup is consistent with the protocol followed in existing graph diffusion papers such as DeFoG \cite{qinmadeira2024defog} which reports the mean and std deviation averaged over 5 sampling runs and not 5 independent runs as it is computationally infeasible. For instance, training DiGress \cite{vignac2023digressdiscretedenoisingdiffusion} on Planar graphs requires nearly 50K epochs with a 10 layer transformer and DeFoG is trained for 100000 epochs with a 10 layer transformer. Additionally, computing the auxiliary features such as spectral adds further overhead. In our setup, we report the mmd ratios averaged over 3 sampling runs.

\begin{table}[h]
\centering
\caption{Hyperparameters used for the graph generation task.}
\label{tab:hyperparams}
\resizebox{10cm}{!}{%
\begin{tabular}{@{}cccccc@{}}
\toprule
Dataset &
  \#Layers &
  LR &
  \begin{tabular}[c]{@{}c@{}}Weight\\ Decay\end{tabular} &
  \begin{tabular}[c]{@{}c@{}}\#Training\\ Epochs\end{tabular} &
  \begin{tabular}[c]{@{}c@{}}Diffusion\\ Steps\end{tabular} \\ \midrule
\multicolumn{6}{c}{JointLSP+DiGress}                       \\ \midrule
CoraChameleon-Node10  & 6 & 0.0001 & 0.00012  & 5000 & 500 \\
CoraChameleon-Node20  & 6 & 0.0001 & 0.00012  & 5000 & 500 \\
CoraChameleon-Node50  & 6 & 0.0001 & 0.00012  & 5000 & 500 \\
CoraChameleon-Node100 & 6 & 0.0001 & 0.00012  & 5000 & 500 \\ \midrule
\multicolumn{6}{c}{JointLSP+SparseDiff}                    \\ \midrule
CoraChameleon-Node10  & 6 & 0.0001 & 0.00001  & 100  & 500 \\
CoraChameleon-Node20  & 6 & 0.0001 & 0.00001  & 100  & 500 \\
CoraChameleon-Node50  & 6 & 0.0001 & 0.00001  & 100  & 500 \\
CoraChameleon-Node100 & 6 & 0.0001 & 0.00001  & 100  & 500 \\
CoraChameleon-Node500 & 4 & 0.0001 & 0.00001  & 100  & 500 \\
EGO                   & 6 & 0.0001 & 0.000012 & 2000 & 500 \\
Community             & 6 & 0.0001 & 0.000012 & 1000 & 500 \\ \bottomrule
\end{tabular}%
}
\end{table}

\paragraph{Runtime analysis}
LSP does not change the complexity of the underlying model. The denoiser forward pass is $O(K n^2 d^2)$ for both DiGress and JointLSP+DiGress; the operations LSP adds is computing $p^0_{ij} = 1 - e^{-e^{u^0_i + u^0_j}}$ across node pairs and the DDPM update on the $n$ latent variables which are $O(n^2)$ and $O(n)$ respectively, smaller than the base cost by a factor $O(K d^2)$. Prior calibration (Algorithm 1) is a one time pre-processing step, not per-iteration. 

\begin{table}[h]
\centering
\caption{Runtime analysis}
\label{tab:runtime}
\resizebox{7cm}{!}{%
\begin{tabular}{@{}ccc@{}}
\toprule
Metric                & DiGress            & JointLSP+DiGress \\ \midrule
Model parameters      & {8,888,579} & 8,888,579        \\
Training step (in ms) & 493.3±4.7        & 489.8±5.0      \\
Peak GPU Memory (GiB) & {13.40}     & {13.39}   \\
Sampling (graph/sec)  & 9.50               & 9.33             \\ \bottomrule
\end{tabular}%
}
\end{table}

\paragraph{List of auxiliary features}
In \ref{tab:digress_feature_names} the details of node level and graph level features used as auxiliary features for DiGress is provided.

\begin{table}[h]
\centering
\caption{Auxiliary features used in DiGress.}
\label{tab:digress_feature_names}
\small
\setlength{\tabcolsep}{6pt}
\renewcommand{\arraystretch}{1.1}
\begin{tabular}{ll}
\toprule
Feature & Description \\
\midrule
\texttt{k3\_node}   & Node triangle count \\
\texttt{k4\_node}   & Node 4-cycle count \\
\texttt{k5\_node}   & Node 5-cycle count \\
\texttt{k3\_graph}  & Graph triangle count \\
\texttt{k4\_graph}  & Graph 4-cycle count \\
\texttt{k5\_graph}  & Graph 5-cycle count \\
\texttt{k6\_graph}  & Graph 6-cycle count \\
\texttt{n\_comp}    & Number of connected components \\
\texttt{nonLCC}     & Indicator for nodes outside the largest connected component \\
\texttt{eigvec1}    & First nontrivial Laplacian eigenvector \\
\texttt{eigvec2}    & Second nontrivial Laplacian eigenvector \\
\texttt{eigval1}    & First nonzero Laplacian eigenvalue \\
\texttt{eigval2}    & Second nonzero Laplacian eigenvalue \\
\texttt{eigval3}    & Third nonzero Laplacian eigenvalue \\
\texttt{eigval4}    & Fourth nonzero Laplacian eigenvalue \\
\texttt{eigval5}    & Fifth nonzero Laplacian eigenvalue \\
\texttt{n/max\_n}   & Normalized graph size ($n / n_{\max}$) \\
\bottomrule
\end{tabular}
\end{table}

\section{Broader impact} \label{broaderimpact}
Graph diffusion models have potential applications across a range of domains, including drug
discovery, protein interaction networks, and social network analysis, where the ability to
generate realistic graph structures could accelerate scientific discovery or inform network design. Like all generative models, these methods are not free from certain risks, for instance, realistic graph
generation could in principle be misused to generate fake synthetic social network structures or in the long term, to design molecules with adversarial properties. However, we note that these
risks are not specific to our contribution. This work addresses a theoretical limitation in
existing discrete graph diffusion pipelines and proposes a prior and a joint diffusion framework that improves
graph generation in the no-node-feature setting.

\clearpage

\clearpage


\end{document}